\documentclass[11pt, reqno]{amsart}
\usepackage[T1]{fontenc}
\usepackage[american]{babel}
\usepackage[utf8]{inputenc}
\usepackage{csquotes}
\usepackage{geometry}
\usepackage{graphicx}
\usepackage[dvipsnames]{xcolor}
\usepackage{hyperref}
\usepackage{enumitem}
\usepackage{amsmath,amssymb,amsthm}
\usepackage{mathtools}  
\usepackage{mathrsfs}
\usepackage{cases}
\usepackage{tikz}

\newcommand{\C}{\mathcal{C}}
\newcommand{\D}{\mathcal{D}}
\newcommand{\E}{\mathbb{E}}
\newcommand{\F}{\mathcal{F}}
\newcommand{\G}{\mathscr{G}}
\renewcommand{\H}{\mathcal{H}}
\newcommand{\K}{\mathcal{K}}
\renewcommand{\L}{\mathcal{L}}
\newcommand{\N}{\mathbb{N}}
\renewcommand{\P}{\mathbb{P}}
\newcommand{\R}{\mathbb{R}}
\newcommand{\W}{\mathsf{W}}
\newcommand{\X}{\mathcal{X}}
\newcommand{\Y}{\mathcal{Y}}

\renewcommand{\restriction}[1]{\raisebox{-.3ex}{$|$}_{#1}}

\DeclareMathOperator{\as}{a.s.}
\DeclareMathOperator{\Var}{Var}
\DeclareMathOperator{\Cov}{Cov}

\newtheoremstyle{plain}{}{}{\itshape}{}{\bfseries}{.}{0.5em}{}
\newtheoremstyle{remark}{}{}{}{}{\itshape}{.}{0.5em}{}

\theoremstyle{plain}
\newtheorem{theorem}{Theorem}
\newtheorem{proposition}{Proposition}
\newtheorem{corollary}{Corollary}
\newtheorem{lemma}{Lemma}

\newtheorem{definition}{Definition}

\theoremstyle{remark}
\newtheorem{remark}{Remark}

\numberwithin{equation}{section}

\title[On the law of the limiting fluctuation process]{
  Quantifying uncertainty in wide two-layer neural networks: on the law of the limiting fluctuation process. 
}

\author{Arnaud Descours\textsuperscript{*}}
\address[*]{Universit\'e Lyon 1, LSAF, EA 2429, Lyon, France}

\author{Arnaud Guillin\textsuperscript{\ensuremath{\dagger}}}
\address[$\dagger$]{Universit\'e Clermont-Auvergne, LMBP, Aubi\`ere, France}

\author{Geoffrey Lacour\textsuperscript{\ensuremath{\ddagger}}}
\address[$\ddagger$]{Université Paris-Saclay, INRAE, MaIAGE, UR 1404, 78350 Jouy-en-Josas, France}

\author{Manon Michel\textsuperscript{\ensuremath{\dagger}}}
\address[$\dagger$]{Universit\'e Clermont-Auvergne, LMBP, Aubi\`ere, France}

\author{Boris Nectoux\textsuperscript{\ensuremath{\dagger}}}
\address[$\dagger$]{Universit\'e Clermont-Auvergne, LMBP, Aubi\`ere, France}

\author{Paul Stos\textsuperscript{\ensuremath{\dagger}}}
\address[$\dagger$]{Universit\'e Clermont-Auvergne, LMBP, Aubi\`ere, France}

\date{}

\begin{document}

\begin{abstract}
  Uncertainty quantification in neural networks prediction is a main
  issue for usual applications. Our approach seeks at reducing
  computation costs by directly evaluating uncertainty using PDE's
  information on the asymptotic variance, rather than the deep
  ensemble method which may be seen as a Monte Carlo estimation of
  the prediction, requiring the training of multiple networks. We thus
  study the law of the limiting process describing the random
  fluctuations around the mean-field limit of wide two-layer neural
  networks trained by stochastic gradient descent in a weak-noise
  regime. Building on a recent trajectorial central limit theorem, in
  which this limit is characterized as the weak solution of a linear
  stochastic evolution equation, we identify its law explicitly. More
  precisely, we show that it is a centered Gaussian process in the
  dual of a weighted Sobolev space, and we derive a closed covariance
  representation for the finite-dimensional distributions obtained by
  testing it against smooth functions. This covariance is expressed
  through the solution of a backward transport equation with a
  nonlocal source term, whose coefficients are driven by the
  mean-field trajectory. 
  As a consequence, by testing against the activation function at a fixed input,
  we obtain an expression for the limiting variance of the corresponding
  network-output fluctuations.
  We illustrate this result numerically on a
  one-dimensional regression example.
\end{abstract}

\maketitle
\section{Introduction}

Neural networks have become standard tools for a wide range of high-dimensional prediction tasks~\cite{goodfellow-bengio-courville-2016,lecun-bengio-hinton-2015,schmidhuber-2015}, and their empirical success has prompted sustained efforts to understand the mathematical mechanisms underlying their behavior~\cite{suh-cheng-2025,shalev-shwartz-ben-david-2014,bach-2024}. 

In this work we focus on the classical supervised learning framework for regression. We observe a stream of independent data points $(x_k,y_k)\in \X\times\Y$ drawn from a common (unknown) distribution $\pi$. The goal is to learn a predictor $\hat{y}$ of the label $y$ that generalizes to previously unseen inputs $x$.
We model this predictor by a two-layer neural network of width $N \ge 1$, namely
\begin{equation}\label{eq:network}
  \hat y_\theta(x)
  =
  \frac1N\sum_{i=1}^N \sigma_*(x,\theta^i),
\end{equation}
where $\theta=(\theta^1,\ldots,\theta^N)\in(\R^d)^N$ denotes the collection of trainable parameters and $\sigma_*:\X\times\R^d\to\R$ is an activation function. The factor $1/N$ is the standard \emph{mean-field} normalization of the network output; it is made for convenience and can be absorbed into the activation by simple redefinition.

In principle, one would like to choose $\theta$ so as to minimize the \emph{population risk}
\[
  \E_\pi\big[\mathsf{L}(\hat y_\theta(x),y)\big],
\]
where $\mathsf{L} : \Y\times\Y \to \R$ is a loss function. We work with the quadratic loss $\mathsf{L}(a,b)=\frac12 |a-b|^2$, as is customary for regression problems. Since the distribution $\pi$ is unknown, the parameters are learned from the observed samples by stochastic gradient descent (SGD) or its variants~\cite{bottou-2010,bottou-curtis-nocedal-2018}.
In the present setting, parameters are initialized i.i.d.\ from a common distribution $\mu_0 \in \mathcal{P}(\R^d)$, and then updated according to noisy SGD with fixed learning rate $\alpha > 0$. More precisely, for every $k\in\N$ and every $i\in\{1,\ldots,N\}$, we consider the \emph{weak-noise} dynamics
\begin{equation}\label{eq:SGD}\tag{SGD}
    \theta^i_{k+1}
    =
    \theta^i_k
    +
    \dfrac{\alpha}{N}
    \bigl(y_k-\hat y_{\theta_k}(x_k)\bigr)
    \nabla_{\theta^i}\sigma_*(x_k,\theta^i_k)
    +
    \dfrac{\varepsilon^i_k}{N^\beta},
\end{equation}
with $\beta>1/2$. 
The random variables $\varepsilon^i_k\sim\mathcal N(0,I_d)$ are assumed independent across neurons and time steps, and the data points $(x_k,y_k)_{k\in\N}$ are i.i.d.\ with common law $\pi$. This idealization is standard in the learning literature and is relevant whenever data are abundant or streaming, so that repeated reuse of a fixed finite dataset is not the leading phenomenon~\cite{bottou-2010,saad-solla-1995}. For ease of exposition, we restrict here to batch size one (also known as \emph{online} or \emph{one-pass} SGD~\cite{bottou-2010}), although the results and arguments below extend to the random mini-batch setting considered in~\cite{descours-guillin-michel-nectoux-2024}. 

Providing a precise and rigorous probabilistic description of the training dynamics of neural networks remains challenging, even in the basic two-layer architecture trained by stochastic gradient descent. For wide two-layer networks, a particularly fruitful approach has been to model training as a large interacting particle system and to track it via the empirical distribution of the neurons' parameters~\cite{chen-rotskoff-bruna-vanden-eijnden-2020,mei-montanari-nguyen-2018,sirignano-spiliopoulos-2020a}, or its rescaled-in-time version:
\begin{equation}\label{eq:empirical-measure}
  \mu^N_t
  =
  \frac1N\sum_{i=1}^N \delta_{\theta^i_{\lfloor Nt\rfloor}}, \qquad t \ge 0.
\end{equation}
As the number of neurons grows to infinity, this empirical process converges in probability (in an appropriate topology) to a deterministic trajectory $\bar\mu=(\bar\mu_t)_{t\ge0}$~\cite{mei-montanari-nguyen-2018,sirignano-spiliopoulos-2020a,descours-guillin-michel-nectoux-2024} characterized as the solution to a nonlinear measure-valued evolution equation (see~\eqref{eq:LLN}). This result can be viewed as a law of large numbers for the training dynamics of~\eqref{eq:SGD}.

Beyond this first-order description, a natural next question would be to understand the random deviations of $\mu^N$ around its mean-field limit $\bar\mu$. In the recent papers~\cite{sirignano-spiliopoulos-2020b,descours-guillin-michel-nectoux-2024}, a trajectorial central limit theorem was derived for the \emph{fluctuation process}
\[
  \eta^N_t
  =
  \sqrt N\,(\mu^N_t-\bar\mu_t), \qquad t \ge 0.
\]
Under the assumptions recalled below, the sequence $(\eta^N)_{N\ge1}$ was found to converge in distribution to a limiting process $\eta^*$, characterized as the weak solution of a linear stochastic evolution equation driven by an infinite-dimensional Gaussian noise (see~\eqref{eq:CLT}). 

While this central limit theorem identifies the fluctuation limit at the level of weak convergence, its law remains rather implicit. The purpose of the present paper is to make this law explicit. To the best of our knowledge, this is the first work to provide a complete characterization of the law of the output fluctuations in this setting.
We prove that the limiting fluctuation process $\eta^*$ is a centered Gaussian process in the dual of a weighted Sobolev space. More concretely, for every finite family of smooth test functions $\varphi_1,\ldots,\varphi_n$, we show that the finite-dimensional process
\[
  \bigl(\langle \varphi_1,\eta^*\rangle,\ldots,
  \langle \varphi_n,\eta^*\rangle\bigr)
\]
is centered Gaussian, and we explicitly characterize the evolution equation for its covariance structure. This formula is expressed through the solution of a backward transport equation~\eqref{eq:T} with a nonlocal source term, whose coefficients are driven by the mean-field trajectory $\bar\mu$. This equation is chosen by duality so as to cancel the deterministic drift terms in the linear stochastic evolution equation~\eqref{eq:CLT} satisfied by $\eta^*$. As a result, for every smooth test function $\varphi$ and every training time $t\ge0$, the random variable $\langle\varphi,\eta^*_t\rangle$ can be represented as the sum of an initial fluctuation term and a Gaussian noise term. The Gaussianity and the covariance formula then follow from the law of the initial fluctuation, the covariance structure of the driving noise, and their independence.

This characterization has a direct interpretation for the network output. Indeed, for any fixed input $x\in\X$, the prediction of the finite-width network after $\lfloor Nt\rfloor$ iterations of~\eqref{eq:SGD} can be written as the empirical-measure observable
\[
  \hat{y}_{\theta_{\lfloor Nt\rfloor}}(x)
  =
  \langle \sigma_*(x,\cdot),\mu^N_t\rangle.
\]
Combining the law of large numbers and the central limit theorem recalled above, we obtain the asymptotic expansion
\[
  \langle \sigma_*(x,\cdot),\mu^N_t\rangle
  \approx
  \langle \sigma_*(x,\cdot),\bar\mu_t\rangle
  +
  \frac1{\sqrt N}
  \langle \sigma_*(x,\cdot),\eta^*_t\rangle.
\]
Taking $\varphi=\sigma_*(x,\cdot)$ in this approximation and applying our covariance formula therefore gives, at least numerically, access to the limiting variance of the order $1/\sqrt{N}$ output fluctuations at the prediction point $x$. 
We illustrate this result on the one-dimensional regression example introduced by Mei, Montanari and Nguyen~\cite{mei-montanari-nguyen-2018}. This example is simple enough to remain numerically tractable while still being rich enough to provide a meaningful test of the predicted variance profile. 
Implementing the variance formula requires solving the backward transport equation~\eqref{eq:T} and repeatedly evaluating nonlocal expectations against the data distribution $\pi$ along the mean-field trajectory. In the experiment of Section~\ref{sec:numerics}, rather than discretizing the deterministic mean-field equation~\eqref{eq:LLN} for $\bar\mu$ directly, we approximate this trajectory by a large-width neural network used as a proxy for the mean-field dynamics. This simplifies the computation, but introduces an additional source of randomness through the proxy trajectory. 
In a fully deterministic pipeline, one would instead discretize the mean-field equation~\eqref{eq:LLN} for $\bar\mu$ directly, and use the resulting trajectory to solve the backward equation~\eqref{eq:T}. Once these two equations have been solved, the variance formula requires no further sampling of independently trained networks. If implemented efficiently, such a procedure would provide a principled route toward computing asymptotic predictive uncertainty for wide neural networks.

\subsection*{Related works} $~$

{\bf Mean field analysis.} The mean-field analysis of wide neural networks has developed rapidly in recent years. For two-layer networks, mean-field limits and their connection with Wasserstein gradient flows were studied, among others, by Mei, Montanari and Nguyen~\cite{mei-montanari-nguyen-2018}, Chizat and Bach~\cite{chizat-bach-2018}, and Rotskoff and Vanden-Eijnden~\cite{rotskoff-vanden-eijnden-2022}. A trajectorial law of large numbers and central limit theorem for stochastic-gradient dynamics on a fixed time interval were obtained by Sirignano and Spiliopoulos~\cite{sirignano-spiliopoulos-2020a, sirignano-spiliopoulos-2020b}, and later extended at the level of the full trajectory (under mini-batching and noisy regimes) by Descours, Guillin, Michel and Nectoux~\cite{descours-guillin-michel-nectoux-2024}. Quantitative propagation-of-chaos estimates for continuous-time counterparts of SGD in wide two-layer networks were also obtained in~\cite{de-bortoli-durmus-fontaine-simsekli-2020}. Related mean-field viewpoints have also been developed beyond shallow networks, including deep neural networks~\cite{sirignano-spiliopoulos-2022, nguyen-pham-2023,fang-lee-yang-zhang-2021}, residual neural networks~\cite{chaintron-chizat-maass-2026}, and transformer models~\cite{rigollet-2025,chen-lin-polyanskiy-rigollet-2025}. The mean field approach has also been considered for bayesian neural networks~\cite{Descours-huix-moulines-guillin-michel-nectoux-2026}, where a law of large numbers and a CLT have been proved. 

{\bf Fluctuation theory of particles system.} The functional-analytic formulation used in this paper is also rooted in the
classical fluctuation theory of mean-field interacting particle systems. In that
setting, fluctuations around a law-of-large-numbers limit are often described by
distribution-valued Gaussian processes solving linear stochastic evolution
equations; see, for instance, Sznitman~\cite{sznitman-1991}, Hitsuda and
Mitoma~\cite{hitsuda-mitoma-1986}, Fernandez and
Méléard~\cite{fernandez-meleard-1997}, Jourdain and
Méléard~\cite{jourdain-meleard-1998}, and Kurtz and Xiong~\cite{kurtz-xiong-2004}.
Hilbertian and Sobolev-dual frameworks provide a natural setting for such
limits, both for tightness arguments and for the characterization of
distribution-valued processes through their action on test functions. This is
precisely the viewpoint adopted here, following in
particular~\cite{fernandez-meleard-1997}.

{\bf Uncertainty quantification in neural networks}. Finally, our covariance formula can be interpreted as an asymptotic characterization of the variability of the trained network output. This connects with the broader literature on uncertainty estimates for neural-network predictions, including early heteroscedastic regression approaches~\cite{nix-weigend-1994} and modern ensemble-based predictive uncertainty methods~\cite{lakshminarayanan-pritzel-blundell-2017}. 
Indeed, deep ensembles are now widely regarded as one of the reference approaches for improving the reliability of neural-network predictions, especially when predictive uncertainty needs to be quantified. Systematically introduced in \cite{lakshminarayanan-pritzel-blundell-2017}, the method consists in training several neural networks independently, typically from different random initializations, and then aggregating their predictive distributions. Despite its conceptual simplicity, this approach has proved highly competitive with more complex Bayesian methods, in particular in terms of probability calibration, robustness, and out-of-distribution detection. Its empirical importance is particularly high in the presence of complex loss landscapes: in \cite{fort2019deep}, the authors argue that different ensemble members may explore distinct regions of parameter space and thereby produce significant functional diversity, which improves the quality of uncertainty estimates. The relevance of deep ensembles is further supported in settings involving data distribution shift: in fact in \cite{ovadia2019can}, they show that deep ensembles are among the most robust methods for maintaining reliable uncertainty estimates when test data differ from training data. These contributions are part of a broader concern with the calibration of modern neural networks, highlighted for a classification task in \cite{guo2017calibration} which shows that highly accurate neural networks may nevertheless output poorly calibrated probabilities. Finally, several variants have been proposed to reduce the computational cost of ensembles. Snapshot ensembles, introduced by~\cite{huang2017snapshot}, 
exploit a cyclic learning-rate schedule to obtain multiple models during a single training run, while BatchEnsemble, proposed by~\cite{wen-tran-ba-2020}, shares most parameters across ensemble members in order to reduce memory and computational overhead. Overall, deep ensembles provide a particularly attractive trade-off between ease of implementation, predictive performance, and robust uncertainty estimation, which explains their frequent use as a strong empirical baseline in modern deep learning.
Our approach could be used to reduce computational costs by directly evaluating uncertainty using PDE's information on the asymptotic variance, rather than a crude Monte Carlo estimation. We illustrate such an application in a simple one-dimensional example in Section~\ref{sec:numerics}.

\subsection*{Outline of the paper}

The rest of the paper is organized as follows. Section~\ref{sec:Sobolev} introduces the weighted Sobolev spaces and embeddings used throughout. In Section~\ref{sec:LLN-CLT}, we recall the two main results from~\cite{descours-guillin-michel-nectoux-2024} that serve as the starting point of this work: a law of large numbers for the trajectory of the empirical measure of the parameters, and a central limit theorem for the associated fluctuation process around the mean-field limit (see Theorems~\ref{thm:LLN} and~\ref{thm:CLT}). Our main result (Theorem~\ref{thm:main}) is stated in Section~\ref{sec:main} and proved in Section~\ref{sec:proof}. The study of the backward transport equation~\eqref{eq:T} is carried out separately in Section~\ref{sec:pde}. Finally, Section~\ref{sec:numerics} contains a numerical illustration in a one-dimensional example.

\section{Functional setting}
\label{sec:Sobolev}

In this section, we introduce the class of weighted Sobolev spaces used throughout the paper and recall the Sobolev embeddings needed later on. 

\subsection{Weighted Sobolev spaces}

Let $E$ be a Banach space, and let $E'$ denote its dual. For any $\mu \in E'$ and $\varphi \in E$, we write $\mu(\varphi)=\langle \varphi,\mu\rangle_E$, or simply $\langle \varphi,\mu\rangle$ when no confusion is possible. For a multi-index $m=(m_1,\ldots,m_d)\in\N^d$, we set
\[
  \partial^m = \partial_1^{m_1}\cdots \partial_d^{m_d},
  \qquad
  |m|=m_1+\cdots+m_d.
\]
Let $\C_c^\infty(\R^d)$ be the space of smooth compactly supported functions on $\R^d$. For $J, j\in\N$ and $\varphi \in \C_c^\infty(\R^d)$, define
\[
  \|\varphi\|_{\H^{J,j}}
  =
  \bigg(
    \sum_{|m|\leq J}
    \int_{\R^d}
    \frac{|\partial^m \varphi(w)|^2}{1+|w|^{2j}}
    \,dw
  \bigg)^{1/2}.
\]
We denote by $\H^{J,j}(\R^d)$ the closure of $\C_c^\infty(\R^d)$ with respect to this norm. Endowed with the induced scalar product, $\H^{J,j}(\R^d)$ is thus a separable Hilbert space. We denote its dual by $\H^{-J,j}(\R^d)$.

\subsection{Sobolev embeddings}

For $J, j\in\N$, let $\C^{J,j}(\R^d)$ denote the space of functions $\varphi \in \C^J(\R^d)$ such that, for every multi-index $m\in\N^d$ with $|m|\leq J$, 
\[
  \lim_{|w|\to\infty}\frac{|\partial^m \varphi(w)|}{1+|w|^j} = 0.
\]
We equip $\C^{J,j}(\R^d)$ with the norm
\[
  \|\varphi\|_{\C^{J,j}}
  =
  \sum_{|m|\leq J}
  \sup_{w\in\R^d}
  \frac{|\partial^m \varphi(w)|}{1+|w|^j}.
\]
We also write $\C_b(\R^d)$ for the space of bounded continuous functions on $\R^d$, endowed with the supremum norm, and $\C_b^\infty(\R^d)$ for the space of smooth functions on $\R^d$ whose derivatives of every order are bounded. Note that for every $J \in \N$, $\C_b^\infty(\R^d) \subset \H^{J,j}(\R^d)$ whenever $j>d/2$. More generally, if $\chi\in \C_c^\infty(\R^d,[0,1])$ is equal to $1$ in a neighborhood of $0$, then $w\mapsto \big(1-\chi(w)\big)|w|^\ell \in \H^{J,j}(\R^d)$ whenever $j-\ell>d/2$.

We shall use the following weighted Sobolev embeddings from~\cite[Section~2]{fernandez-meleard-1997}.

\begin{proposition}\label{prop:embedding}
Let $J, K, j, k \in\N$ with $K, k > d/2$. Then
\[
  \H^{J+K,j}(\R^d)\hookrightarrow \C^{J,j}(\R^d)
  \quad\text{and}\quad
  \H^{J+K,j}(\R^d) \hookrightarrow \H^{J,j+k}(\R^d).
\]
\end{proposition}

\section{Law of large numbers and central limit theorem}
\label{sec:LLN-CLT}

In this section, we recall the law of large numbers and central limit theorem proved in~\cite{descours-guillin-michel-nectoux-2024} for the trajectory $\mu^N = (\mu^N_t)_{t \ge 0}$ of the scaled empirical measure of the parameters of a wide two-layer neural network. These two results form the starting point of our analysis. Accordingly, throughout this paper, we work in the framework of~\cite{descours-guillin-michel-nectoux-2024} and adopt the notation introduced there, which we summarize below.

\subsection{Assumptions}

Let $E$ be a metric space. We denote by $\D(\R_+,E)$ the space of càdlàg functions from $\R_+=[0,\infty)$ to $E$, endowed with the Skorokhod topology. Note that for every $j \in \N$ and every $N \ge 1$, $\mu^N$ is a random element of $\D\big(\R_+,\C^{0,j}(\R^d)'\big)$. Hence, by Proposition~\ref{prop:embedding}, it is also a random element of $\D\big(\R_+,\H^{-K,j}(\R^d)\big)$ for every $K  > d/2$.

Let $\mathcal{P}_p(\R^d)$ denote the Wasserstein space of order $p \ge 1$, namely
\[
  \mathcal{P}_p(\R^d)
  =
  \Big\{
  \mu \in \mathcal{P}(\R^d), \ \int_{\R^d} |w|^p\,\mu(dw) < \infty
  \Big\},
\]
endowed with the distance
\[
  \W_p(\mu,\nu)
  =
  \Big(\inf\big\{ \E\big[|X-Y|^p\big], \ \L(X)=\mu,\ \L(Y)=\nu \big\}\Big)^{1/p}.
\]
Recall that $(\mathcal{P}_p(\R^d), \W_p)$ is Polish~\cite[Chapter~5]{santambrogio-2015}. Note that for every $p \ge 1$ and every $N \ge 1$, $\mu^N$ is also a random element of $\D\big(\R_+,\mathcal{P}_p(\R^d)\big)$.

For $N \ge 1$, we introduce the filtration
\begin{equation}\label{eq:filtration}
  \begin{aligned}
    \F_0^N &= \sigma\bigl(\theta_0^1, \dots, \theta_0^N\bigr), \\[0.3em]
    \F_k^N &=
    \sigma\bigl(\theta_0^i,\,(x_\ell,y_\ell),\,\varepsilon_\ell^i,\
    0\le \ell\le k-1,\ 1\le i\le N\bigr), \quad k \ge 1.
  \end{aligned}
\end{equation}

We also fix the following Sobolev exponents:
\begin{equation}\label{eq:exponents}
  L=\lceil d/2\rceil+3, 
  \qquad 
  \gamma=4\lceil d/2\rceil+5, 
  \qquad
  J_0 \ge 4\lceil d/2\rceil + 8, 
  \qquad 
  j_0 = \lceil d/2\rceil + 2.
\end{equation}

We shall work under the following assumptions, inherited from~\cite{descours-guillin-michel-nectoux-2024}:
\begin{enumerate}[label={(A\arabic*)}, ref={(A\arabic*)}, font=\bfseries, itemsep=3pt, topsep=3pt]
    \item\label{assump:A1} The activation function $\sigma_* :  \X \times \R^d \to \R$ belongs to $\C_b^\infty(\X \times \R^d)$.
    \item\label{assump:A2} The sequence of data points $(x_k,y_k)_{k \in \N}$ is i.i.d.\ with common law $\pi \in \mathcal{P}(\X \times \Y)$ such that $\E\big[|y|^{16(\gamma+1)}\big] < \infty$. Moreover, for every $k \in \N$, $(x_k,y_k)$ is independent of $\F_k^N$.
    \item\label{assump:A3} The initial parameters $\theta_0^1, \dots, \theta_0^N$ are i.i.d.\ with common law $\mu_0 \in \mathcal{P}(\R^d)$, and $\mu_0$ is compactly supported.
    \item\label{assump:A4} The noise variables $(\varepsilon_k^1, \dots, \varepsilon_k^N)_{k\in \N}$ are i.i.d.\ with law $\mathcal{N}(0,I_d)$. Moreover, for every $k \in \N$, $(\varepsilon_k^1, \dots, \varepsilon_k^N)$ is independent of $\F_k^N$.
\end{enumerate}

\begin{remark}
  Assumption~\ref{assump:A1} is made for convenience; see Remark 5 after Theorem 1 in~\cite{descours-guillin-michel-nectoux-2024} for a discussion of the regularity assumptions on the activation function.
\end{remark}

\subsection{Law of large numbers}

The next theorem establishes the convergence in probability of the sequence $(\mu^N)_{N \geq 1}$ in $\D\big(\R_+, \mathcal{P}_{\gamma}(\R^d)\big)$ towards a deterministic continuous trajectory $\bar\mu$.

\begin{theorem}[\cite{descours-guillin-michel-nectoux-2024}]\label{thm:LLN}
Assume~\ref{assump:A1}--\ref{assump:A4} and $\beta > 1/2$. Then the sequence $(\mu^N)_{N \ge 1}$ converges in probability in $\D\big(\R_+,\mathcal{P}_\gamma(\R^d)\big)$ to a deterministic element $\bar\mu \in \C\big(\R_+,\mathcal{P}_1(\R^d)\big)$ characterized as the unique solution in $\C\big(\R_+,\mathcal{P}_1(\R^d)\big)$ to the following nonlinear measure-valued evolution equation: 
\begin{equation}\label{eq:LLN}
  \begin{aligned}
    &\forall \varphi \in \C_b^\infty(\R^d),\ \forall t \ge 0, \\
    &\qquad 
    \langle\varphi,\bar\mu_t\rangle = \langle \varphi,\mu_0\rangle
    + \alpha \int_0^t \E_\pi\Big[
      \bigl(y-\langle \sigma_*(x, \cdot),\bar\mu_s\rangle\bigr)
      \big\langle \nabla \varphi \cdot \nabla \sigma_*(x, \cdot),\bar\mu_s\big\rangle
    \Big]
    \,ds.
  \end{aligned}
\end{equation}
In addition, the limit $\bar\mu$ also belongs to $\C\big(\R_+,\H^{-L,\gamma}(\R^d)\big)$, and equation~\eqref{eq:LLN} remains valid for test functions in $\H^{L,\gamma}(\R^d)$.
\end{theorem}

\subsection{Central limit theorem}

From now on, we assume that Assumptions~\ref{assump:A1}--\ref{assump:A4} hold and $\beta > 3/4$. For $N \ge 1$, we define the \emph{fluctuation process} $\eta^N = (\eta^N_t)_{t \ge 0}$ by
\[
  \eta_t^N = \sqrt{N}\,(\mu_t^N-\bar\mu_t), \qquad t \ge 0,
\]
where $\bar\mu \in \H^{-L,\gamma}(\R^d)$ is the limit trajectory of $(\mu^N)_{N \ge 1}$ given by Theorem~\ref{thm:LLN}. Since $L > d/2$ (see~\eqref{eq:exponents}), the empirical measure process $\mu^N$ takes values in $\D\big(\R_+,\H^{-L,\gamma}(\R^d)\big)$, and hence the fluctuation process $\eta^N$ is a random element of $\D\big(\R_+,\H^{-L,\gamma}(\R^d)\big)$. 
Moreover, the embedding $\H^{J_0-1,j_0}(\R^d)\hookrightarrow \H^{L,\gamma}(\R^d)$ from Proposition~\ref{prop:embedding} also allows us to view $\eta^N$ as a random element of $\D\big(\R_+,\H^{-J_0+1,j_0}(\R^d)\big)$.

The second main result of~\cite{descours-guillin-michel-nectoux-2024} establishes a limit theorem for the fluctuation process $\eta^N$ as $N \to \infty$. The limiting fluctuation process $\eta^*$ is characterized there as the unique \emph{weak} solution of the stochastic evolution equation~\eqref{eq:CLT} driven by a so-called \emph{G-process} (see Definition~\ref{def:G-process} below). In this case, uniqueness holds \emph{in distribution}, leading to the notion of a \emph{weak} solution, recalled in Definition~\ref{def:weak}. 

\begin{definition}[G-process]\label{def:G-process}
We say that a process $\G \in \C\big(\R_+,\H^{-J_0,j_0}(\R^d)\big)$ is a \emph{G-process} if, for every $n \ge 1$ and every $\varphi_1,\dots,\varphi_n \in \H^{J_0,j_0}(\R^d)$, the finite-dimensional process \[
  \bigl(\langle \varphi_1,\G\rangle,\dots,\langle \varphi_n,\G\rangle\bigr)\in \C(\R_+,\R^n)
\]
has zero mean, independent Gaussian increments (w.r.t.\ its natural filtration), and covariance structure given by
\begin{equation}\label{eq:cov-G}
  \begin{aligned}
    &\forall\,i,j \in \{1,\dots,n\},\ \forall\,0 \le s \le t,\\
    &\qquad
    \Cov\bigl(\langle \varphi_i,\G_t\rangle,\langle \varphi_j,\G_s\rangle\bigr)
    =
    \alpha^2 \int_0^s
    \Cov_\pi\bigl(Q_r[\varphi_i],Q_r[\varphi_j]\bigr)\,dr,
  \end{aligned}
\end{equation}
where, for $\varphi \in \H^{J_0,j_0}(\R^d)$ and $r \ge 0$, 
\[
  Q_r[\varphi] = 
  \bigl(y-\langle \sigma_*(x, \cdot),\bar\mu_r\rangle\bigr)
  \big\langle \nabla \varphi \cdot \nabla \sigma_*(x, \cdot),\bar\mu_r\big\rangle.
\]
Finally, given a filtration $\mathcal{F}$, we say that $\G$ is an \emph{$\mathcal{F}$-G-process} if, in addition, its  increments are independent of $\mathcal{F}$ (i.e., for all $0\le s\le t$, $\G_t-\G_s$ is independent of $\mathcal{F}_s$). 
\end{definition}

\begin{remark}
Note that $Q_r[\varphi]$ is well defined for every $\varphi \in \H^{J_0,j_0}(\R^d)$ since for each $i \in \{1,\dots,d\}$, $\partial_i \varphi \in \H^{J_0-1,j_0}(\R^d) \hookrightarrow \H^{L,\gamma}(\R^d)$ by Proposition~\ref{prop:embedding}.
\end{remark}

\begin{remark}
Recall from~\cite[Proposition~34]{descours-guillin-michel-nectoux-2024} that the law of a $\C\big(\R_+,\H^{-J_0,j_0}(\R^d)\big)$-valued process $\nu$ is fully determined by the laws of all the finite-dimensional processes $\bigl(\langle \varphi_1,\nu\rangle,\dots,\langle \varphi_n,\nu\rangle\bigr)$, $\varphi_1,\dots,\varphi_n \in \H^{J_0,j_0}(\R^d), \ n\ge 1$. Therefore, Definition~\ref{def:G-process} uniquely specifies the law of $\G$. 
\end{remark}

\begin{remark}\label{rmk:F-martingale}
Note that for every G-process $\G$, for every $n \ge 1$ and $\varphi_1,\dots,\varphi_n \in \H^{J_0,j_0}(\R^d)$, the process $\big(\langle \varphi_1,\G\rangle,\dots,\langle \varphi_n,\G \rangle\big)$ is a martingale with respect to its natural filtration. 
If, in addition, $\G$ is an $\mathcal{F}$-G-process, then $\big(\langle \varphi_1,\G\rangle,\dots,\langle \varphi_n,\G\rangle\big)$ is also an $\mathcal{F}$-martingale whenever it is adapted to $\mathcal F$.
\end{remark}

Let $\eta$ be a $\C\big(\R_+,\H^{-J_0+1,j_0}(\R^d)\big)$-valued process and let $\G \in \C\big(\R_+,\H^{-J_0,j_0}(\R^d)\big)$ be a G-process. We consider the stochastic evolution equation:
\begin{equation}\label{eq:CLT}
  \begin{aligned}
    &\as\quad \forall \varphi \in \H^{J_0,j_0}(\R^d),\ \forall t \ge 0,\\
    &\qquad
    \langle \varphi,\eta_t\rangle
    =
    \langle \varphi,\eta_0\rangle + \langle \varphi, \G_t\rangle
    +
    \alpha \int_0^t \E_\pi\Big[
      \bigl(y-\langle \sigma_*(x, \cdot),\bar\mu_s\rangle\bigr)
      \big\langle \nabla \varphi \cdot \nabla \sigma_*(x, \cdot),\eta_s\big\rangle \\
      &\hspace*{7.5cm} 
      -\langle \sigma_*(x, \cdot),\eta_s\rangle
      \big\langle \nabla \varphi \cdot \nabla \sigma_*(x, \cdot),\bar\mu_s\big\rangle 
    \Big]    
    \,ds.\\
  \end{aligned}
\end{equation}

For a process $X$, we denote by $\mathcal{F}^X$ its natural filtration. 

\begin{definition}[Weak solution]\label{def:weak}
Let $\nu$ be an $\H^{-J_0+1,j_0}(\R^d)$-valued random variable. We say that a process $\eta \in \C\big(\R_+,\H^{-J_0+1,j_0}(\R^d)\big)$ is a \emph{weak solution} of~\eqref{eq:CLT} with initial distribution $\nu$ if there exists an $\mathcal F^{\eta}$-G-process $\G \in \C\big(\R_+,\H^{-J_0,j_0}(\R^d)\big)$ such that~\eqref{eq:CLT} holds and $\eta_0 = \nu$ in distribution. In addition, we say that \emph{weak uniqueness} holds if any two weak solutions of~\eqref{eq:CLT} (possibly defined on different probability spaces) with the same initial distribution have the same law.
\end{definition}

\begin{remark}\label{rmk:independence}
Let us record two consequences of the definition. First, if $\G$ is an $\mathcal F^\eta$-G-process, then $\G$ is independent of the initial condition $\eta_0$. Second, if $\eta$ is a weak solution of~\eqref{eq:CLT}, then, for every $\varphi\in\H^{J_0,j_0}(\R^d)$, the real-valued process $\langle\varphi,\G\rangle$ is $\mathcal F^\eta$-adapted and hence an $\mathcal F^\eta$-martingale, by Remark~\ref{rmk:F-martingale}. Since $\H^{-J_0,j_0}(\R^d)$ is separable, this also implies that $\G_t$ is $\mathcal F_t^\eta$-measurable for every $t\ge0$.
\end{remark}

We can now state the central limit theorem for the fluctuation process.

\begin{theorem}[\cite{descours-guillin-michel-nectoux-2024}]\label{thm:CLT}
Assume~\ref{assump:A1}--\ref{assump:A4} and $\beta > 3/4$. Then the sequence $(\eta^N)_{N \ge 1}$ converges in distribution in $\D\big(\R_+,\H^{-J_0+1,j_0}(\R^d)\big)$ to a process $\eta^* \in \C\big(\R_+,\H^{-J_0+1,j_0}(\R^d)\big)$ characterized as the unique weak solution of~\eqref{eq:CLT} with initial distribution $\nu_0$ such that, for every $n \ge 1$ and every $\varphi_1,\dots,\varphi_n \in \H^{J_0-1,j_0}(\R^d)$,
\begin{equation}\label{eq:initial}
  \bigl(\langle \varphi_1,\nu_0\rangle,\dots,\langle \varphi_n,\nu_0\rangle\bigr)
  \sim
  \mathcal N\bigl(0,\Gamma(\varphi_1,\dots,\varphi_n)\bigr),
\end{equation}
where $\Gamma(\varphi_1,\dots,\varphi_n)$ is the covariance matrix of $\bigl(\varphi_1(\theta_0^1),\dots,\varphi_n(\theta_0^1)\bigr)$.
\end{theorem}

Let us mention that the independence between the initial condition $\eta^*_0$ and the associated G-process $\G^*$ driving equation~\eqref{eq:CLT} from Definition~\ref{def:weak} is essential (and far from trivial) for invoking uniqueness in law for equation~\eqref{eq:CLT} based on the strong uniqueness established in~\cite[Section~3.5.2]{descours-guillin-michel-nectoux-2024} (see indeed~\cite[Theorem~32.14]{kallenberg-2002}). We note that this property was not accounted for in the proof of Theorem 2 in~\cite{descours-guillin-michel-nectoux-2024}; this gap will be addressed in Section~\ref{sec:proof}, as it is also fundamental for determining the covariance formula~\eqref{eq:covariance} of the limiting fluctuation process.

\section{Main Result}
\label{sec:main}

The main result of this work (Theorem~\ref{thm:main} below) gives an explicit characterization of the evolution of the law of the limiting fluctuation process $\eta^*$ from Theorem~\ref{thm:CLT}. Its covariance structure is described via the solution of a backward transport equation. More precisely, given $t\ge0$ and $\varphi \in \C_b^\infty(\R^d)$, we consider the backward problem on $[0,t]\times \R^d$:
\begin{equation}\label{eq:T}\tag{T}
  \begin{aligned}
    &\partial_s f(s,w)
    + \alpha\,\E_\pi\Big[
      \big(y-\langle \sigma_*(x,\cdot),\bar\mu_s\rangle\big)
      \nabla f(s,w)\cdot \nabla \sigma_*(x,w) \\
      &\hspace*{3.85cm}-
      \big\langle \nabla f(s) \cdot \nabla \sigma_*(x,\cdot),\bar\mu_s\big\rangle\,
      \sigma_*(x,w)
    \Big]
    =0,\\
    &f(t,w) = \varphi(w).
  \end{aligned}
\end{equation}
Here and throughout, for a function $f:(s,w)\in[0,t]\times\R^d\mapsto f(s,w)$, we write $f(s)$ (or $f(s,\cdot)$) for the function $w\mapsto f(s,w)$. By Theorem~\ref{thm:pde} in Section~\ref{sec:pde}, for every final time $t\ge0$ and terminal datum $\varphi\in \C_b^\infty(\R^d)$, equation~\eqref{eq:T} admits a unique solution denoted by $f^\varphi$ in $\C^1\big([0,t],\H^{J_0,j_0}(\R^d)\big)\cap\C^1\big([0,t]\times\R^d\big)$.

We are now in position to state our main result.

\begin{theorem}\label{thm:main}
Assume~\ref{assump:A1}--\ref{assump:A4} and $\beta>3/4$.
Then for every $n \ge 1$ and every $\varphi_1,\dots,\varphi_n\in \C_b^\infty(\R^d)$, the finite-dimensional process 
\[
  \big(\langle\varphi_1, \eta^*\rangle, \dots, \langle\varphi_n, \eta^*\rangle\big) \in \C(\R_+,\R^n)
\]
is a centered Gaussian process. Moreover, for any $\varphi,\psi\in \C_b^\infty(\R^d)$ and any $0 \le s \le t$,
\begin{equation}\label{eq:covariance}
  \begin{aligned}
      \Cov\big(\langle \varphi,\eta^*_s\rangle,\langle \psi,\eta^*_t\rangle\big)
      ={}
      &\Cov\big(f^\varphi(0,\theta_0^1),f^\psi(0,\theta_0^1)\big)\\
      &\qquad
      +\alpha^2\int_0^s
      \Cov_\pi\big(Q_r[f^\varphi(r)],Q_r[f^\psi(r)]\big)\,dr,
  \end{aligned}
\end{equation}
where $f^\varphi$ and $f^\psi$ denote respectively the solutions of~\eqref{eq:T} on $[0,s]$ and $[0,t]$ with terminal data $f^\varphi(s,\cdot)=\varphi$ and $f^\psi(t,\cdot)=\psi$.
\end{theorem}

In particular, for every $t\ge 0$ and every
$\varphi\in \C_b^\infty(\R^d)$, it holds
\begin{equation}\label{eq:variance}
  \Var\big(\langle \varphi,\eta^*_t\rangle\big)
  =
  \Var\big(f^\varphi(0,\theta_0^1)\big)
  +\alpha^2\int_0^t
  \Var\big(Q_s[f^\varphi(s)]\big)\,ds.
\end{equation}
Taking $\varphi=\sigma_*(x,\cdot)$ in~\eqref{eq:variance} thus yields the explicit variance formula for the network output at the prediction point $x\in \X$ in the mean-field regime. We illustrate this formula numerically in Section~\ref{sec:numerics}.

\begin{remark}
Although Theorem~\ref{thm:main} is stated for test functions in $\C_b^\infty(\R^d)$, this is enough to characterize the law of the $\H^{-J_0+1,j_0}(\R^d)$-valued process $\eta^*$. Indeed, $\C_c^\infty(\R^d)\subset \C_b^\infty(\R^d)$ is dense in $\H^{J_0-1,j_0}(\R^d)$, and the duality pairing with elements of $\H^{-J_0+1,j_0}(\R^d)$ is continuous on $\H^{J_0-1,j_0}(\R^d)$. Hence the finite-dimensional laws obtained by testing against smooth test functions in $\C_b^\infty(\R^d)$ determine those obtained by testing against arbitrary functions in the separable Hilbert space $\H^{J_0-1,j_0}(\R^d)$, and therefore completely determine the law of the whole process $\eta^*$. 
\end{remark}

\begin{remark}
\end{remark}

\section{Proof of Theorem~\ref{thm:main}}
\label{sec:proof}

This section is devoted to the proof of Theorem~\ref{thm:main}. To this end, we assume that Assumptions~\ref{assump:A1}--\ref{assump:A4} hold and that $\beta>3/4$. The analysis of the backward equation~\eqref{eq:T} is deferred to Section~\ref{sec:pde}, and we shall use its well-posedness results throughout the proof.

\subsection{A preliminary result}

The first step in the proof of Theorem~\ref{thm:main} is to establish that the driving noise $\G^*$ in equation~\eqref{eq:CLT} satisfied by the limiting fluctuation process $\eta^*$ is an $\mathcal{F}^{\eta^*}$-G-process (see Definition~\ref{def:G-process}). Since the fact that $\G^*$ is a G-process has already been established in~\cite{descours-guillin-michel-nectoux-2024}, we focus here on proving that its increments are independent of the filtration $\mathcal{F}^{\eta^*}$. Let us emphasize that while~\cite{descours-guillin-michel-nectoux-2024} proved that $\G^*$ is a G-process, the stronger property of $\mathcal{F}^{\eta^*}$-independence was not explicitly addressed. Verifying this independence fills a technical gap in the proof of~\cite[Theorem~2]{descours-guillin-michel-nectoux-2024}, as it is a necessary condition to invoke uniqueness in law for the weak solution of~\eqref{eq:CLT}. It is also an essential ingredient for the derivation of the explicit covariance formula~\eqref{eq:covariance} in Theorem~\ref{thm:main}.

Following the notation of~\cite{descours-guillin-michel-nectoux-2024}, for
$\varphi \in \H^{J_0,j_0}(\R^d)$, $N \ge 1$, and $k \in \N$, we introduce
\begin{align*}
 \langle \varphi, M^N_k\rangle
 &= \frac{\alpha}{N}\big(y-\langle \sigma_*(x,\cdot),\mu^N_{k/N}\rangle\big)
    \big\langle \nabla \varphi \cdot \nabla \sigma_*(x,\cdot),\mu^N_{k/N}\big\rangle \\
 &\qquad
 - \frac{\alpha}{N}\E_\pi\Big[
   \big(y-\langle \sigma_*(x,\cdot),\mu^N_{k/N}\rangle\big)
   \big\langle \nabla \varphi \cdot \nabla \sigma_*(x,\cdot),\mu^N_{k/N}\big\rangle
   \Big],
\end{align*}
and, for all $t\ge 0$,
\[
 \langle \varphi, M^N_t\rangle
 = \sum_{k=0}^{\lfloor Nt\rfloor-1}\langle \varphi,M^N_k\rangle
\]
(with the usual convention $\sum_0^{-1}=0$). We recall from~\cite[Propositions~24 and~27]{descours-guillin-michel-nectoux-2024} that the sequence $(\eta^N,\sqrt{N} M^N)_{N\ge 1}$ is tight in the product space
\begin{equation}\label{eq:E}
  \mathcal{E}
  =
  \D\big(\R_+,\H^{-J_0+1,j_0}(\R^d)\big)
  \times
  \D\big(\R_+,\H^{-J_0,j_0}(\R^d)\big).
\end{equation}
Let $(\eta, \G)$ be one of its limit points along some subsequence, that we still denote by $N$ for ease of notation. It follows from~\cite[Propositions~32 and~36]{descours-guillin-michel-nectoux-2024} that $\G$ is a G-process (see Definition~\ref{def:G-process}) and $\eta$ solves~\eqref{eq:CLT} driven by $\G$ and with initial distribution $\nu_0$.

\begin{proposition}\label{prop:independence}
Assume~\ref{assump:A1}--\ref{assump:A4} and  $\beta>3/4$.  Then, the increments of the G-process $\G \in \C(\R_+,\H^{-J_0,j_0}(\R^d))$  are independent of $\mathcal F^\eta$. In particular, $\G$ is an $\mathcal F^\eta$-G-process and $\eta$ is a weak solution of~\eqref{eq:CLT} in the sense of Definition~\ref{def:weak}.
\end{proposition}

\begin{proof}
\emph{Step 1.} Fix $n \ge 1$ and $\varphi_1,\ldots,\varphi_n \in \H^{J_0,j_0}(\R^d)$. For $t\ge 0$, set $\mathfrak{F}_t^N = \F^N_{\lfloor Nt\rfloor}$ (see~\eqref{eq:filtration}) and define for $N \ge 1$, 
\begin{equation}\label{eq:JN}
  J^N
  =
  \big(
  \langle \varphi_1,\sqrt{N}M_t^N\rangle,\ldots,
  \langle \varphi_n,\sqrt{N}M_t^N\rangle
  \big)_{t \ge 0}
  \in \D(\R_+,\R^n),
\end{equation}
which is an $\mathfrak{F}_t^N$-martingale by~\cite[Lemma~21]{descours-guillin-michel-nectoux-2024}.
Set also
\[
  X
  =
  \big(
  \langle \varphi_1,\G_t\rangle,\ldots,
  \langle \varphi_n,\G_t\rangle
  \big)_{t \ge 0}
  \in \C(\R_+,\R^n).
\]
In view of the proof of~\cite[Proposition~32]{descours-guillin-michel-nectoux-2024}, to show Proposition~\ref{prop:independence}, it is enough to prove that
the increments of the process $X$ are independent of $\mathcal F^\eta$. 

\medskip
\noindent
\emph{Step 2.}
Fix $\psi\in \C_c^\infty(\R)$ and $u\in\R^n$, and define for $t \ge 0$
\[
  Y_t^u
  =
  \psi(u\cdot X_t)
  -\frac12\int_0^t \psi''(u\cdot X_s)\,dc_u(s),
\]
where $c_u(t)= u\cdot C(t)\,u$ and $C(t) \in \R^{n \times n}$ is the square matrix defined by
\[
  C_{ij}(t)
  =
  \alpha^2\int_0^t
  \Cov_\pi\big(Q_s[\varphi_i],Q_s[\varphi_j]\big)\,ds,
  \qquad 1 \le i,j \le n.
\]
We claim that the following holds: for any fixed $0\le s\le t$, 
$h\in \C_b\big(\D\big([0,s],\H^{-J_0+1,j_0}(\R^d)\big), \R\big)$ and 
$g\in \C_b\big(\D([0,s],\R^n), \R\big)$,
\begin{equation}\label{eq:conditional}
  \E\Big[
    h(\eta\restriction{[0,s]})g\big(X\restriction{[0,s]}\big)
    \big(Y_t^u-Y_s^u\big)
  \Big]
  =0,
\end{equation}
where
\[
  X\restriction{[0,s]} = (X_r)_{0 \le r\le s} \in \D\big([0,s], \R^n)
\]
denotes the restriction of the process $X$ to the time interval $[0,s]$ (and similarly for $\eta\restriction{[0,s]}$). Then the process $Y^u$ is a martingale with respect to the joint filtration $\mathfrak{F}$ generated by $(\eta, X)$
(i.e., $\mathfrak{F}_s = \sigma(\eta_r,\,X_r,\ 0\le r\le s)$).
Following the proof of~\cite[Theorem 7.1.2]{ethier-kurtz-2009}, it follows that for all $0\le s\le t$,
\begin{equation}\label{eq:increments}
  \E\big[e^{iu\cdot(X_t-X_s)}\mid \mathfrak{F}_s\big]
  =
  e^{
  -\frac12u\cdot(C(t)-C(s))\,u
  }.
\end{equation}
In turn, equation~\eqref{eq:increments} classically implies that $X_t-X_s$ is independent of $\mathfrak{F}_s$. Since $\F_s^\eta \subset \mathfrak{F}_s$, $X_t-X_s$ is in particular independent of $\F_s^\eta$.  This would conclude the proof of Proposition~\ref{prop:independence}.


\medskip
\noindent
\emph{Step 3.} It remains to show that condition~\eqref{eq:conditional} holds. Fix $0\le s\le t$, and let
$h\in \C_b\big(\D\big([0,s],\H^{-J_0+1,j_0}(\R^d)\big), \R\big)$ and
$g\in \C_b\big(\D([0,s],\R^n), \R\big)$.
For $N\ge 1$ and $t \ge 0$, define
\[
  Y_t^{u,N}
  =
  \psi(u\cdot J_t^N)
  -\frac12\int_0^t \psi''\big(u\cdot J_{r-}^N\big)\,dc_u^N(r),
\]
where $c_u^N(t)=u\cdot C^N(t)\,u$ and $C^N(t) \in \R^{n \times n}$ is the square matrix defined by
\[
  C_{ij}^N(t)
  =
  \big[\langle \varphi_i,\sqrt{N} M^N\rangle,\langle \varphi_j,\sqrt{N} M^N\rangle\big]_t,
  \qquad 1 \le i,j \le n,
\]
i.e., the quadratic covariation between $\langle \varphi_i,\sqrt{N} M_t^N\rangle$ and $\langle \varphi_j,\sqrt{N} M_t^N\rangle$. We will prove that
\begin{equation}\label{eq:Vitali}
  \lim_{N\to\infty}\E[h(\eta^N\restriction{[0,s]})\,g\big(J^N\restriction{[0,s]}\big) (Y_t^{u, N} - Y_s^{u,N})]
  =
  \E\Big[h(\eta\restriction{[0,s]})g\big(X\restriction{[0,s]}\big)\big(Y_t^u-Y_s^u\big)\Big]
\end{equation}
and then show that the left-hand side converges to $0$. 
Set
\[
  \begin{aligned}
    Z_u^N &= \psi(u\cdot J_t^N)-\psi(u\cdot J_s^N) -\frac12 \int_s^t \psi''(u\cdot J_r^N)\,dc_u(r),\\
    R_u^N &=
    \frac12\int_s^t \psi''(u\cdot J_{r-}^N)\,dc_u^N(r)
    -
   \frac12\int_s^t \psi''(u\cdot J_r^N)\,dc_u(r),
  \end{aligned}
\]
so that $Y_t^{u,N} - Y_s^{u,N} = Z_u^N - R_u^N$.
Since $(\eta^N,\sqrt{N}M^N)$ converges in distribution to $(\eta,\G)$ in $\mathcal{E}$ (see~\eqref{eq:E}), $\eta$ and $\G$ have continuous sample paths, and $c_u$ is deterministic and continuous, it follows by the continuous mapping theorem that 
\[
  \big(
    \eta^N|_{[0,s]},\,
    J^N|_{[0,t]},\,
    Z_u^N
  \big)
  \stackrel{d}{\longrightarrow}
  \big(
       \eta|_{[0,s]},\,
    X|_{[0,t]},\,
    Y_t^u - Y_s^u
  \big).
\]
Moreover, using the convergence of the predictable covariations established in
the proof of \cite[Proposition~31]{descours-guillin-michel-nectoux-2024}, namely for all $1 \le i, j \le n$ and $0 \le r \le t$,
\[
  C_{ij}^N(r) \stackrel{\P}{\longrightarrow} C_{ij}(r),
\]
together with the jump estimates from~\cite[(118)]{descours-guillin-michel-nectoux-2024}
\[
  \lim_{N \to\infty} \E\big[\sup_{0 \le r \le t} \Delta J_r^N \big] = 0,
\]
where $\Delta J_r^N = |J_r^N - J_{r-}^N|$, we obtain (see Equation~(1.39) in the proof of~\cite[Theorem~7.1.4 (a)]{ethier-kurtz-2009} for details) that
\[
  \lim_{N \to \infty} 
  \E\Big[\Big|
  \int_s^t \psi''(u\cdot J_{r-}^N)\,dc_u^N(r)
  -
  \int_s^t \psi''(u\cdot J_r^N)\,dc_u(r)
  \Big|\Big] = 0.
\]
Thus $R_u^N \to 0$ in $L^1$, and hence also in probability. By Slutsky's theorem, 
\[
  \big(
      \eta^N|_{[0,s]},\,
    J^N|_{[0,t]},\,
    Z_u^N, R_u^N
  \big)
  \stackrel{d}{\longrightarrow}
  \big(
        \eta|_{[0,s]},\,
    X|_{[0,t]},\,
    Y_t^u - Y_s^u, 0
  \big).
\]
Since $Y_t^{u,N} - Y_s^{u,N} = Z_u^N - R_u^N$, it follows by another application of the continuous mapping theorem
that
\[
  h(    \eta^N|_{[0,s]})g\big(J^N|_{[0,s]}\big)
  \big(Y_t^{u,N}-Y_s^{u,N}\big)
  \stackrel{d}{\longrightarrow}
  h(    \eta|_{[0,s]})g\big(X|_{[0,s]}\big)
  \big(Y_t^u-Y_s^u\big).
\]

To pass to expectations, it suffices to show that the family 
\[
  \Big(h(\eta^N|_{[0,s]})g\big(J^N\restriction{[0,s]}\big)\big(Y_t^{u,N} - Y_s^{u,N}\big)\Big)_{N \ge 1}
\]
is uniformly integrable. Since the functions $h$, $g$ and $\psi$ are bounded, and 
\[
  |Y_t^{u,N} - Y_s^{u,N}| \le 2 \|\psi\|_\infty + \frac12 \|\psi''\|_\infty \big(c_u^N(t) - c_u^N(s)\big),
\]
this amounts to proving uniform integrability of the family $\big(c_u^N(t) - c_u^N(s)\big)_{N \ge 1}$. 

Note that $c_u^N$ is precisely the quadratic covariation of the real-valued pure jump process $u \cdot J^N$, with jumps occurring exactly at times $1/N, \dots, \lfloor Nt \rfloor / N$ (see~\eqref{eq:empirical-measure} and~\eqref{eq:JN}). 
Thus, we have (see, e.g.~\cite[Remark~7.1.5]{ethier-kurtz-2009}) 
\[
  c_u^N(t)-c_u^N(s)
  =
  \sum_{k=\lfloor Ns \rfloor+1}^{\lfloor Nt \rfloor} |u\cdot \Delta J_{k/N}^N|^2.
\]
By definition of $J^N$~\eqref{eq:JN}, for every $k\in\{1,\dots,\lfloor Nt\rfloor\}$,
\[
  \Delta J_{k/N}^N
  =
  \big(
  \langle \varphi_1,\sqrt{N}M_{k-1}^N\rangle,\dots,\langle \varphi_n,\sqrt{N}M_{k-1}^N\rangle
  \big).
\]
Hence, by the Cauchy--Schwarz inequality and convexity of $x \mapsto x^{3/2}$, we have
\[
  |u\cdot \Delta J_{k/N}^N|^3
  \le
  |u|^3\,|\Delta J_{k/N}^N|^3
  \le
  \sqrt{n}\,|u|^3
  \sum_{i=1}^n |\langle \varphi_i,\sqrt{N}M_{k-1}^N\rangle|^3.
\]
The same computations as in the proof of~\cite[Lemma~9~(iii)]{descours-guillin-michel-nectoux-2024} show that, for each $i\in\{1,\dots,n\}$,
\begin{equation}\label{eq:cube}
  \max_{1 \le k\le \lfloor Nt \rfloor} N^{3/2}\,\E\big[|\langle \varphi_i,\sqrt{N}M_{k-1}^N\rangle|^3\big] < \infty.
\end{equation}
Thus, 
\begin{align*}
  \E\big[|c_u^N(t)-c_u^N(s)|^{3/2}\big]
  &\le
  \big(\lfloor Nt\rfloor-\lfloor Ns\rfloor\big)^{1/2}
  \sum_{k=\lfloor Ns\rfloor+1}^{\lfloor Nt\rfloor}
  \E\big[|u\cdot \Delta J_{k/N}^N|^3\big] \\
  &\le
  \sqrt{n} \,t^{3/2}\,|u|^3 \sum_{i=1}^n \max_{1\le k\le \lfloor Nt \rfloor} N^{3/2}\,\E\big[|\langle \varphi_i,\sqrt{N}M_{k-1}^N\rangle|^3\big] < \infty,
\end{align*}
uniformly in $N \ge 1$. Hence, the family $\big(c_u^N(t)-c_u^N(s)\big)_{N \ge 1}$ is bounded in $L^{3/2}$, and in particular, uniformly integrable. This proves~\eqref{eq:Vitali}.

On the other hand, since $\eta^N|_{[0,s]}$, $g\big(J^N\restriction{[0,s]}\big)$ and $Y_s^{u,N}$ are $\mathfrak{F}_s^N$-measurable, it holds
\[
  \E\Big[h(\eta^N|_{[0,s]})g\big(J^N\restriction{[0,s]}\big)\big(Y_t^{u,N}-Y_s^{u,N}\big)\Big]
  =
  \E\Big[
  h(\eta^N|_{[0,s]})g\big(J^N\restriction{[0,s]}\big)
  \big(\E[Y_t^{u,N}\mid\mathfrak{F}_s^N]-Y_s^{u,N}\big)
  \Big].
\]
Let $\Psi_u(w) = \psi(u \cdot w)$, $w \in \R^n$.
Using a second-order Taylor expansion at each jump (see also Itô's formula for semimartingales~\cite[Section 7 in Part II]{protter-2004}), we have
\begin{align*}
  \Big|\E[Y_t^{u,N}\mid\mathfrak{F}_s^N]-Y_s^{u,N}\Big|
  &\le \E\Big[
    \sum_{k=\lfloor Ns\rfloor+1}^{\lfloor Nt\rfloor}
    \big|
      \Psi_u(J_{k/N}^N)-\Psi_u(J_{k/N-}^N)
      -\nabla \Psi_u(J_{k/N-}^N)\cdot \Delta J_{k/N}^N
      \\
      &\hspace{7em}
      -\frac12\,\Delta J_{k/N}^N\cdot \nabla^2\Psi_u(J_{k/N-}^N)\Delta J_{k/N}^N
    \big|
    \;\Big|\; \mathfrak{F}_s^N 
  \Big] \\
  &\le
  \frac16 \|\psi^{(3)}\|_\infty |u|^3\,
  \E\Big[
  \sum_{k=\lfloor Ns\rfloor+1}^{\lfloor Nt\rfloor}
  |\Delta J_{k/N}^N|^3
  \;\Big|\;\mathfrak{F}_s^N
  \Big].
\end{align*}
Taking expectations and using the estimate~\eqref{eq:cube} obtained above on the jumps of $J^N$, we conclude that
\[
\lim_{N\to\infty}
\E\Big[\big|\E[Y_t^{u,N}\mid \mathfrak{F}_s^N]-Y_s^{u,N}\big|\Big]
=0.
\]
Since $h$ and $g$ are bounded, it follows that
\[
\lim_{N\to\infty}
\E\Big[h(\eta^N|_{[0,s]})g\big(J^N\restriction{[0,s]}\big)(Y_t^{u,N}-Y_s^{u,N})\Big]
=0.
\]
Combining this with \eqref{eq:Vitali}, we obtain 
\[
  \E\big[h(\eta|_{[0,s]})g\big(X\restriction{[0,s]}\big)\big(Y_t^u-Y_s^u\big)\big]=0,
\] 
which is exactly \eqref{eq:conditional}. The proof of Proposition~\ref{prop:independence} is complete.
\end{proof}

\begin{corollary}\label{cor:independence}
Assume~\ref{assump:A1}--\ref{assump:A4} and $\beta>3/4$.
Then, the sequence $(\eta^N)_{N\ge1}$ converges in distribution in $\D\big(\R_+, \H^{-J_0+1,j_0}(\R^d)\big)$ to the unique weak solution to~\eqref{eq:CLT} (see Definition~\ref{def:weak}).
\end{corollary}

\begin{proof}
Let $(\eta, \G)$ be a limit point of the sequence $(\eta^N,\sqrt{N}M^N)_{N \ge 1}$, as introduced above (see~\eqref{eq:E}). Recall that by~\cite[Propositions~32 and~36]{descours-guillin-michel-nectoux-2024}, the process $\G$ is a G-process driving equation~\eqref{eq:CLT} satisfied by $\eta$ with initial distribution $\nu_0$. Moreover, by Proposition~\ref{prop:independence}, $\G$ is an $\mathcal{F}^\eta$-G-process, and hence $\eta$ is a weak solution of~\eqref{eq:CLT} in the sense of Definition~\ref{def:weak}.
By strong uniqueness for~\eqref{eq:CLT}, proved in~\cite[Section~3.5.2]{descours-guillin-michel-nectoux-2024}, weak uniqueness holds in the class of solutions from Definition~\ref{def:weak}, since the driving G-process is independent of the initial condition and its law is uniquely determined by Definition~\ref{def:G-process}. Therefore all limit points have the same law, and the whole sequence $(\eta^{N})_{N\ge 1}$ converges in distribution to this unique weak solution.
\end{proof}

From now on, we denote by $\eta^*$ the limiting fluctuation process from Theorem~\ref{thm:CLT}, namely the weak solution of~\eqref{eq:CLT}, and by $\G^*$ the associated $\mathcal F^{\eta^*}$-G-process driving this equation.

\subsection{Extension of the limit equation to time-dependent test functions}
\label{sec:extension}

The next step in the proof of Theorem~\ref{thm:main} is to extend the weak formulation~\eqref{eq:CLT} to test functions depending on both time and space. This is made precise by the following proposition.

\begin{proposition}\label{prop:extension}
  Assume~\ref{assump:A1}--\ref{assump:A4} and $\beta > 3/4$. Let $t\ge0$.
  Then, almost surely, for every $s\in[0,t]$ and every
  $f\in \C^1\big([0,t],\H^{J_0,j_0}(\R^d)\big)$,
  \begin{equation}\label{eq:extension}
  \begin{aligned}
    \langle f(s),\eta_s^*\rangle
    &= \langle f(0),\eta_0^*\rangle 
    +\langle f(s),\G_s^*\rangle 
    - \int_0^s \langle \partial_r f(r),\G_r^*\rangle\,dr 
    + \int_0^s \langle \partial_r f(r),\eta_r^*\rangle\,dr \\
    &\quad + \alpha \int_0^s \E_\pi\Big[
    \big(y-\langle \sigma_*(x,\cdot),\bar\mu_r\rangle\big)
    \big\langle \nabla f(r)\cdot \nabla\sigma_*(x,\cdot),\eta_r^*\big\rangle \\
    &\hspace*{3cm} - 
    \langle \sigma_*(x,\cdot),\eta_r^*\rangle
    \big\langle \nabla f(r)\cdot \nabla\sigma_*(x,\cdot),\bar\mu_r\big\rangle
    \Big]\,dr.
  \end{aligned}
\end{equation}
\end{proposition}

\begin{proof}

\emph{Step 1.} Let $h\in \H^{J_0,j_0}(\R^d)$ and $g\in \C^1([0,t])$. Applying~\eqref{eq:CLT} to the
fixed spatial test function $h$ and integrating by parts in time gives, for
every $s\in[0,t]$,
\begin{align*}
  \langle g(s)h,\eta_s^*\rangle-\langle g(0)h,\eta_0^*\rangle
  &= \langle g(s)h,\G_s^*\rangle  
  - \int_0^s \langle g'(r)h,\G_r^*\rangle\,dr 
  + \int_0^s \langle g'(r)h,\eta_r^*\rangle\,dr \\
  &\quad + \alpha\int_0^s g(r)\,\E_\pi\Big[
  \big(y-\langle \sigma_*(x,\cdot),\bar\mu_r\rangle\big)
  \langle \nabla h\cdot\nabla\sigma_*(x,\cdot),\eta_r^*\rangle \\
  &\hspace*{3.7cm} - 
  \langle \sigma_*(x,\cdot),\eta_r^*\rangle
  \langle \nabla h\cdot\nabla\sigma_*(x,\cdot),\bar\mu_r\rangle
  \Big]\,dr.
\end{align*}
This is exactly~\eqref{eq:extension} with $f(r)=g(r)h$. By linearity,
the same identity holds almost surely for every $s \in [0,t]$ and every function of the form
\[
  f \in D = 
  \Big\{
    r \in [0,t] \mapsto \sum_{i=1}^k g_i(r)h_i,\
    k \ge 1, \
    g_i\in \C^1([0,t]),\ 
    h_i\in \H^{J_0,j_0}(\R^d)
  \Big\}
  .
\]
We will extend this identity to arbitrary functions in $\C^1\big([0,t],\H^{J_0,j_0}(\R^d)\big)$ by density of $D$ in $\C^1\big([0,t],\H^{J_0,j_0}(\R^d)\big)$ and passage to the limit. 

\medskip
\noindent
\emph{Step 2.} In this step, we prove that $D$ is dense in $\C^1\big([0,t],\H^{J_0,j_0}(\R^d)\big)$. Fix $f \in \C^1\big([0,t],\H^{J_0,j_0}(\R^d)\big)$ and let $(e_n)_{n\ge1}$ be an orthonormal basis of $\H^{J_0,j_0}(\R^d)$. For $k\ge1$, let $P_k:\H^{J_0,j_0}(\R^d)\to \H^{J_0,j_0}(\R^d)$ denote the orthogonal projection onto $\mathrm{span}(e_1,\dots,e_k)$, and define 
\[
  f_k = P_k f: r \in [0,t] \longmapsto \sum_{i=1}^k \langle f(r),e_i\rangle_{\H^{J_0,j_0}} \,e_i.
\]
Note that for each $i \ge 1$, the map $r \mapsto \langle f(r),e_i\rangle$ belongs to
$\C^1([0,t])$, with derivative
$\frac{d}{dr}\langle f(r),e_i\rangle
=
\langle \partial_r f(r),e_i\rangle$.
Hence, for every $k \ge 1$, $f_k\in D$, and $\partial_r f_k =P_k\partial_r f$. 
By Parseval's identity, for every $h\in \H^{J_0,j_0}(\R^d)$, $P_k h$ converges to $h$ in $\H^{J_0,j_0}(\R^d)$. Moreover, since $P_k$ is an orthogonal projection, its operator norm is uniformly bounded by $1$. Hence the convergence is actually uniform on compact subsets of $\H^{J_0,j_0}(\R^d)$, namely 
\[
  \sup_{h \in \K} \|P_kh - h\|_{\H^{J_0,j_0}} \longrightarrow 0
\]
as $k \to \infty$ for every compact subset $\K \subset \H^{J_0,j_0}(\R^d)$.
Since $f$ and $\partial_r f$ are continuous from the compact time interval $[0,t]$ into $\H^{J_0,j_0}(\R^d)$, the sets $\K_0 = f([0,t])$ and $\K_1 = \partial_r f([0,t])$ are compact in $\H^{J_0,j_0}(\R^d)$. Therefore,
\[
  \sup_{r\in[0,t]}
  \|f_k(r)-f(r)\|_{\H^{J_0,j_0}}
  =
  \sup_{h \in \K_0}
  \|P_kh-h\|_{\H^{J_0,j_0}}
  \longrightarrow 0,
\]
and, since $\partial_r f_k=P_k\partial_r f$,
\[
  \sup_{r\in[0,t]}
  \|\partial_r f_k(r)-\partial_r f(r)\|_{\H^{J_0,j_0}}
  =
  \sup_{h\in\K_1}
  \|P_kh-h\|_{\H^{J_0,j_0}}
  \longrightarrow 0.
\]
Consequently, we have $f_k\to f$ in
$\C^1\big([0,t],\H^{J_0,j_0}(\R^d)\big)$, which proves that $D$ is dense in $\C^1\big([0,t],\H^{J_0,j_0}(\R^d)\big)$.

\medskip
\noindent
\emph{Step 3.} We now pass to the limit from test functions in $D$
to arbitrary functions in $\C^1\big([0,t],\H^{J_0,j_0}(\R^d)\big)$. For $f \in \C^1\big([0,t],\H^{J_0,j_0}(\R^d)\big)$, define the random linear map $\Lambda=\Lambda(\omega)$ by
\[
\begin{aligned}
  \forall s \in [0,t], \quad
  \Lambda f(s)
  ={}&
  \big\langle f(s),\eta_s^*\big\rangle
  -\big\langle f(0),\eta_0^*\big\rangle
  -\big\langle f(s),\G_s^*\big\rangle\\
  &
  -\int_0^s \big\langle \partial_r f(r),\eta_r^*\big\rangle\,dr 
  +\int_0^s \big\langle \partial_r f(r),\G_r^*\big\rangle\,dr\\
  &-\alpha\int_0^s
  \E_\pi\Big[
    \big(y-\langle\sigma_*(x,\cdot),\bar\mu_r\rangle\big)
    \big\langle \nabla f(r)\cdot\nabla\sigma_*(x,\cdot),\eta_r^*\big\rangle\\
  &\hspace*{2.6cm}-
    \big\langle\sigma_*(x,\cdot),\eta_r^*\big\rangle
    \big\langle \nabla f(r)\cdot\nabla\sigma_*(x,\cdot),\bar\mu_r\big\rangle
  \Big]\,dr.
\end{aligned}
\]
Let $\Omega^*$ be the full-probability set on which both the trajectories of $\eta^*$ and $\G^*$ are continuous and identity~\eqref{eq:extension} holds for all $s\in[0,t]$ and all functions in $D$. For every $\omega\in\Omega^*$, it is straightforward to check that $\Lambda(\omega)$ is a continuous linear map from $\C^1\big([0,t],\H^{J_0,j_0}(\R^d)\big)$ to $\C([0,t],\R)$. By the two previous steps, $\Lambda(\omega)$ vanishes on the dense subset $D$, hence it is identically zero on $\C^1\big([0,t],\H^{J_0,j_0}(\R^d)\big)$. Therefore $\Lambda(\omega)\equiv0$ on $\Omega^*$, which is precisely~\eqref{eq:extension}. This concludes the proof of Proposition~\ref{prop:extension}.
\end{proof}

\subsection{Proof of Theorem~\ref{thm:main}}

We now combine Proposition~\ref{prop:extension} with the backward
equation~\eqref{eq:T} studied in Section~\ref{sec:pde}.
Fix $t \ge 0$ and let $\varphi\in \C_b^\infty(\R^d)$. Let $f^\varphi$ denote the solution of~\eqref{eq:T} in $\C^1\big([0,t],\H^{J_0,j_0}(\R^d)\big) \cap \C^1\big([0,t]\times\R^d\big)$ with terminal condition $f^\varphi(t,\cdot)=\varphi$, given by
Theorem~\ref{thm:pde}. Applying Proposition~\ref{prop:extension} to $f=f^\varphi$ and using that $f^\varphi$ solves equation~\eqref{eq:T}, we obtain
\begin{equation}\label{eq:representation}
  \as\quad \langle \varphi,\eta_t^*\rangle
  =
  \langle f^\varphi(0),\eta_0^*\rangle
  + Z_t^\varphi,
\end{equation}
where we set
\[
  Z_t^\varphi
  =
  \langle \varphi,\G_t^*\rangle
  -\int_0^t \langle \partial_s f^\varphi(s),\G_s^*\rangle\,ds.
\]
Using this representation, we prove Theorem~\ref{thm:main} by treating separately the initial term $\langle f^\varphi(0),\eta_0^*\rangle$ and the noise contribution $Z_t^\varphi$, and then using their independence, given by Proposition~\ref{prop:independence} and Remark~\ref{rmk:independence}, to conclude.

\begin{proof}[Proof of Theorem~\ref{thm:main}]
\emph{Step 1.} Let $n \ge 1$ and $\varphi_1,\dots,\varphi_n\in \C_b^\infty(\R^d)$. We begin by proving that
$\big(\langle \varphi_1,\eta^*\rangle,\dots,\langle \varphi_n,\eta^*\rangle\big)$
is a centered Gaussian process in $\C(\R_+,\R^n)$.
To this end, let $p\ge1$ and fix times $0 \le t_1 \le \dots \le t_p$. Let us show that
\[
  \big(\langle \varphi_i,\eta_{t_j}^*\rangle\big)_
  {\substack{1\le i\le n,\ 1\le j\le p}}
\]
is a centered Gaussian vector.

For clarity, in this step we make explicit the dependence on the terminal time and write
$f_{t_j}^{\varphi_i}$ for the solution of~\eqref{eq:T} on $[0,t_j]$ with terminal datum $f_{t_j}^{\varphi_i}(t_j,\cdot)=\varphi_i$. Let
$(a_{i,j})_{1\le i\le n,\ 1\le j\le p}\subset\R$. By
\eqref{eq:representation}, it holds that
\[
\sum_{j=1}^p\sum_{i=1}^n a_{i,j}
\langle \varphi_i,\eta_{t_j}^*\rangle
=
\Big\langle
\sum_{j=1}^p\sum_{i=1}^n a_{i,j}
f_{t_j}^{\varphi_i}(0),\eta_0^*
\Big\rangle
+
\sum_{j=1}^p\sum_{i=1}^n a_{i,j}Z_{t_j}^{\varphi_i}.
\]
The first term on the right-hand side is a centered Gaussian random variable by the description of the initial
fluctuation law in Theorem~\ref{thm:CLT}, since each
$f_{t_j}^{\varphi_i}(0)$ belongs to $\H^{J_0,j_0}(\R^d)\hookrightarrow
\H^{J_0-1,j_0}(\R^d)$. Moreover, by Corollary~\ref{cor:independence} and Remark \ref{rmk:independence}, this initial term is independent of $\G^*$, and hence of all the variables
$Z_{t_j}^{\varphi_i}$, $1\le i\le n$, $1\le j\le p$.

Thus, it remains to check that 
\[
  Y=\sum_{j=1}^p\sum_{i=1}^n a_{i,j}Z_{t_j}^{\varphi_i}
\]
is a centered Gaussian random variable.
For $m\ge1$, define 
\[
  Y_m
  =
  \sum_{j=1}^p\sum_{i=1}^n
  a_{i,j}
  \bigg(
    \langle \varphi_i,\G_{t_j}^*\rangle
    -
    \frac{t_j}{m}
    \sum_{q=1}^m
    \big\langle
      \partial_s f_{t_j}^{\varphi_i}\Big(\frac{q t_j}{m}\Big),
      \G_{q t_j/m}^*
    \big\rangle
  \bigg).
\]
Since, by definition, $\G^*$ is almost surely continuous, for each pair of indices $(i,j)$, the map
$s\mapsto\langle \partial_s f_{t_j}^{\varphi_i}(s),\G_s^*\rangle$ is almost surely continuous on $[0,t_j]$. Hence $Y_m\to Y$ almost surely, and in particular, in distribution, as $m\to\infty$.
Let us now show that each $Y_m$ is a centered Gaussian random variable. Fix $m\ge1$ and consider the
finite family

\[
  F_m=\big\{\varphi_i, \ \partial_s f_{t_j}^{\varphi_i}(q t_j/m), \
  1\le i\le n,\ 1\le j\le p,\ 1\le q\le m\big\} \subset \H^{J_0,j_0}(\R^d).
\]
By Definition~\ref{def:G-process}, the process
\[
  s \in\R_+ \longmapsto \big(\langle h,\G_s^*\rangle\big)_{h\in F_m}
\]
is a finite-dimensional centered Gaussian process. 
In particular, its evaluations at the finitely many times $t_j$ and $q t_j/m$, $1\le j\le p$, $1\le q\le m$, form a centered Gaussian vector. Since $Y_m$ is a linear combination of the coordinates of this vector, $Y_m$ is a centered Gaussian random variable.

Finally, by the stability of Gaussian laws under weak convergence (see, e.g.,~\cite[Proposition~1.1]{le-gall-2016}), $Y$ is again a centered Gaussian random variable. Together with the Gaussianity of the initial term and the independence established above (see Remark~\ref{rmk:independence}), this concludes the proof of the Gaussianity statement.

\medskip
\noindent
\emph{Step 2.} We now turn to the proof of the covariance formula~\eqref{eq:covariance}.
Let $\varphi,\psi\in \C_b^\infty(\R^d)$ and fix $0 < s\le t$. Using~\eqref{eq:representation} and the independence between $\eta_0^*$ and $\G^*$ from Proposition~\ref{prop:independence} and Remark~\ref{rmk:independence}, it holds
\begin{equation}\label{eq:split}
  \Cov\big(\langle \varphi,\eta_s^*\rangle,\langle \psi,\eta_t^*\rangle\big)
  =
  \Cov\big(
  \langle f^\varphi(0),\eta_0^*\rangle,
  \langle f^\psi(0),\eta_0^*\rangle
  \big) +
  \Cov\big(Z_s^\varphi,Z_t^\psi\big).
\end{equation}
By the covariance structure~\eqref{eq:initial} of the initial fluctuation in Theorem~\ref{thm:CLT}, we have
\[
  \Cov\big(
  \langle f^\varphi(0),\eta_0^*\rangle,
  \langle f^\psi(0),\eta_0^*\rangle
  \big)
  =
  \Cov\big(f^\varphi(0,\theta_0^1),f^\psi(0,\theta_0^1)\big).
\]

It remains to identify the covariance $\Cov\big(Z_s^\varphi,Z_t^\psi\big)$ of the noise terms in~\eqref{eq:split}. Recall that
\[
  Z_s^\varphi
  =
  \langle \varphi,\G_s^*\rangle
  -
  \int_0^s
  \langle \partial_u f^\varphi(u),\G_u^*\rangle\,du,
  \qquad
  Z_t^\psi
  =
  \langle \psi,\G_t^*\rangle
  -
  \int_0^t
  \langle \partial_v f^\psi(v),\G_v^*\rangle\,dv.
\]
Since all these random variables are centered, we have 
\[
\begin{aligned}
  \Cov(Z_s^\varphi,Z_t^\psi)
  &=
  \E\big[\langle \varphi,\G_s^*\rangle
          \langle \psi,\G_t^*\rangle\big] \\
  &\quad
  - \E\bigg[
      \langle \varphi,\G_s^*\rangle
      \int_0^t
      \langle \partial_v f^\psi(v),\G_v^*\rangle\,dv
    \bigg] \\
  &\quad
  - \E\bigg[
      \langle \psi,\G_t^*\rangle
      \int_0^s
      \langle \partial_u f^\varphi(u),\G_u^*\rangle\,du
    \bigg] \\
  &\quad
  + \E\bigg[
      \int_0^s
      \langle \partial_u f^\varphi(u),\G_u^*\rangle\,du
      \int_0^t
      \langle \partial_v f^\psi(v),\G_v^*\rangle\,dv
    \bigg].
\end{aligned}
\]
We now simplify this expression.
All applications of Fubini's theorem below are justified by the covariance formula~\eqref{eq:cov-G}, Assumption~\ref{assump:A2} and the fact that $f^\varphi\in \C^1\big([0,s],\H^{J_0,j_0}(\R^d)\big)$ and $f^\psi\in \C^1\big([0,t],\H^{J_0,j_0}(\R^d)\big)$.
Using the covariance structure of $\G^*$~\eqref{eq:cov-G}, we have
\[
  \E\big[\langle \varphi,\G_s^*\rangle
          \langle \psi,\G_t^*\rangle\big]
  =
  \alpha^2
  \int_0^s
  \Cov_\pi\big(
    Q_r[\varphi],
    Q_r[\psi]
  \big)\,dr.
\]
Moreover, by linearity of $Q_r[\cdot]$ and since $f^\psi(t)=\psi$, 
\[
\begin{aligned}
  \E\bigg[
      \langle \varphi,\G_s^*\rangle
      \int_0^t
      \langle \partial_v f^\psi(v),\G_v^*\rangle\,dv
    \bigg]
  &=
  \int_0^t
  \E\big[
    \langle \varphi,\G_s^*\rangle
    \langle \partial_v f^\psi(v),\G_v^*\rangle
  \big]\,dv \\
  &=
  \alpha^2
  \int_0^t
  \int_0^{s\wedge v}
  \Cov_\pi\big(
    Q_r[\varphi],
    Q_r[\partial_v f^\psi(v)]
  \big)\,dr\,dv\\
  &= 
  \alpha^2 
  \int_0^s
  \Cov_\pi\bigg(
    Q_r[\varphi], \int_r^t
    Q_r[\partial_v f^\psi(v)]\,dv
  \bigg)\,dr \\
  &=
  \alpha^2 
  \int_0^s
  \Cov_\pi\Big(
    Q_r[\varphi], Q_r[\psi]-Q_r[f^\psi(r)]
  \Big)\,dr.
\end{aligned}
\]
Using the same computations and the terminal datum $f^\varphi(s)=\varphi$, we have
\[
\begin{aligned}
  \E\bigg[
      \langle \psi,\G_t^*\rangle
      \int_0^s
      \langle \partial_u f^\varphi(u),\G_u^*\rangle\,du
    \bigg]
  &=
  \alpha^2
  \int_0^s
  \Cov_\pi\big(
    Q_r[\varphi]-Q_r[f^\varphi(r)],
    Q_r[\psi]
  \big)\,dr.
\end{aligned}
\]
Similarly, for the last term,
\[
\begin{aligned}
  &\E\bigg[
      \int_0^s
      \langle \partial_u f^\varphi(u),\G_u^*\rangle\,du
      \int_0^t
      \langle \partial_v f^\psi(v),\G_v^*\rangle\,dv
    \bigg] \\
  &\qquad =
  \alpha^2
  \int_0^s
  \Cov_\pi\bigg(
    \int_r^s Q_r[\partial_u f^\varphi(u)]\,du,
    \int_r^t Q_r[\partial_v f^\psi(v)]\,dv
  \bigg)\,dr \\
  &\qquad =
  \alpha^2
  \int_0^s
  \Cov_\pi\big(
    Q_r[\varphi]-Q_r[f^\varphi(r)],
    Q_r[\psi]-Q_r[f^\psi(r)]
  \big)\,dr.
\end{aligned}
\]
Combining these four terms, the covariance reduces to
\[
\begin{aligned}
  \Cov(Z_s^\varphi,Z_t^\psi)
  &=
  \alpha^2
  \int_0^s
  \Big[
    \Cov_\pi(Q_r[\varphi],Q_r[\psi]) \\
  &\qquad\qquad\quad
    -\Cov_\pi\big(
      Q_r[\varphi],
      Q_r[\psi]-Q_r[f^\psi(r)]
    \big) \\
  &\qquad\qquad\quad
    -\Cov_\pi\big(
      Q_r[\varphi]-Q_r[f^\varphi(r)],
      Q_r[\psi]
    \big) \\
  &\qquad\qquad\quad
    +\Cov_\pi\big(
      Q_r[\varphi]-Q_r[f^\varphi(r)],
      Q_r[\psi]-Q_r[f^\psi(r)]
    \big)
  \Big]\,dr \\
  &=
  \alpha^2
  \int_0^s
  \Cov_\pi\big(
    Q_r[f^\varphi(r)],
    Q_r[f^\psi(r)]
  \big)\,dr.
\end{aligned}
\]
Combining this identity with~\eqref{eq:split} yields
\[
\begin{aligned}
  \Cov\big(
    \langle \varphi,\eta_s^*\rangle,
    \langle \psi,\eta_t^*\rangle
  \big)
  &=
  \Cov\big(
    f^\varphi(0,\theta_0^1),
    f^\psi(0,\theta_0^1)
  \big) \\
  &\qquad+
  \alpha^2
  \int_0^s
  \Cov_\pi\big(
    Q_r[f^\varphi(r)],
    Q_r[f^\psi(r)]
  \big)\,dr,
\end{aligned}
\]
which is exactly~\eqref{eq:covariance}. The proof of Theorem~\ref{thm:main} is complete.
\end{proof}

\section{Backward transport equation}
\label{sec:pde}

In this section, we study the backward transport equation~\eqref{eq:T} associated with Theorem~\ref{thm:main}, which we recall here for convenience: for $t \ge 0$ and $\varphi \in \C_b^\infty(\R^d)$, we consider the backward problem on $[0,t] \times \R^d$:
\begin{equation}\tag{T}
  \begin{aligned}
    &\partial_s f(s,w)
    + \alpha\,\E_\pi\Big[
      \big(y-\langle \sigma_*(x,\cdot),\bar\mu_s\rangle\big)
      \nabla f(s,w)\cdot \nabla \sigma_*(x,w) \\
      &\hspace*{3.85cm}-
      \big\langle \nabla f(s) \cdot \nabla \sigma_*(x,\cdot),\bar\mu_s\big\rangle\,
      \sigma_*(x,w)
    \Big]
    =0,\\
    &f(t,w) = \varphi(w).
  \end{aligned}
\end{equation}
We keep throughout this section Assumptions~\ref{assump:A1}--\ref{assump:A4}. We also only assume that $\beta>1/2$ since only the law of large numbers (Theorem~\ref{thm:LLN}) is needed here.

\begin{theorem}
\label{thm:pde}
Assume~\ref{assump:A1}--\ref{assump:A4} and $\beta > 1/2$. For any final time $t \ge 0$ and final datum $\varphi \in \C_b^\infty(\R^d)$, equation~\eqref{eq:T} admits a unique mild solution $f^\varphi \in \C\big([0,t],\H^{J_0,j_0}(\R^d)\big)$ in the sense of Definition~\ref{def:mild-T}. Moreover,
\[
  f^\varphi \in \C^1\big([0,t],\H^{J_0,j_0}(\R^d)\big) \cap \C^1\big([0,t]\times \R^d\big)
\]
and therefore $f^\varphi$ is a classical solution to the backward transport equation~\eqref{eq:T}. 
Finally, the mapping $\varphi\in\H^{J_0,j_0}(\R^d)
\mapsto f^\varphi\in \C\big([0,t],\H^{J_0,j_0}(\R^d)\big)$ is linear and Lipschitz-continuous. 
\end{theorem}

\begin{remark}
  The proof of Theorem~\ref{thm:pde} indicates that the solution $f^\varphi$ belongs in fact to the space $\C^1\big([0,t], \H^{J,j}(\R^d)\big)$ for any $J \in \N$ provided $j > d/2$. Moreover, for every terminal datum $\varphi\in \H^{J_0,j_0}(\R^d)$, equation~\eqref{eq:T} still admits a unique mild solution $f^\varphi$ in $\C\big([0,t],\H^{J_0,j_0}(\R^d)\big)$ (see indeed the first step in the proof of Theorem~\ref{thm:pde} below).
\end{remark}

The proof follows a standard fixed-point argument based on the Duhamel formula~\eqref{eq:Duhamel} along the characteristic flow of the transport part introduced below.

\subsection{Characteristic flow} Fix $t \ge 0$. For $s\in[0,t]$ and $w\in\R^d$, set
\[
  F(s,w)
  =
  \alpha\,\E_\pi\Big[
  \big(y-\langle \sigma_*(x,\cdot),\bar\mu_s\rangle\big)
  \nabla \sigma_*(x,w)
  \Big]
  \in \R^d,
\]
and, for a function $g \in\C^1\big([0,t]\times \R^d\big)$, define
\[
  Ag(s,w)
  =
  -\alpha\,\E_\pi\Big[
  \langle \nabla g(s)\cdot \nabla \sigma_*(x,\cdot),\bar\mu_s\rangle
  \,\sigma_*(x,w)
  \Big].
\]
Then, Equation~\eqref{eq:T} rewrites as
\begin{equation}
  \label{eq:backward-transport-equation}
  \begin{cases}
    \partial_s f + F\cdot \nabla f + Af = 0
    & \text{in } [0,t]\times\R^d,\\
    f(t)=\varphi & \text{in } \R^d.
  \end{cases}
\end{equation}
For fixed terminal datum $\varphi \in \C_b^\infty(\R^d)$ and $g:[0,t]\times \R^d\to\R$, we also consider the linear transport equation with frozen source term
\begin{equation}
  \label{eq:frozen}
  \begin{cases}
    \partial_s f + F\cdot \nabla f = -Ag
    & \text{in } [0,t]\times\R^d,\\
    f(t)=\varphi & \text{in } \R^d.
  \end{cases}
\end{equation}

Let us now introduce the characteristic flow associated with the transport part of~\eqref{eq:backward-transport-equation}. For any fixed $w\in \R^d$, let $\phi_t(\cdot,w)$ be the solution to the following (non-autonomous) ordinary differential equation in $\R^d$:
\begin{equation}\label{eq:flow}
  \begin{cases}
  \partial_s \psi(s)=F\big(s,\psi(s)\big),\\
  \psi(t)=w.
  \end{cases}
\end{equation}
Since $F$ is globally Lipschitz on $\R_+ \times \R^d$ (this follows easily from the estimates in the proof of~\cite[Proposition~19]{descours-guillin-michel-nectoux-2024}), the flow is well defined globally on $[0,t]$ for every $w \in \R^d$, and satisfies
\[
  \forall s\in[0,t],\quad
  \phi_t(s,w) = w - \int_s^t F_r\big(\phi_t(r,w)\big)\,dr.  
\]
Note however that we cannot \emph{a priori} define this equation for $s < 0$ as the function $F(s)$ is not defined for negative times. 

We collect below some properties of the flow $\phi$ that will be needed later on. For any $t_1, t_2, t_3\ge 0$, by uniqueness of the Cauchy problem on $\R^d$
\[
  \begin{cases}
  \partial_s \psi(s)=F\big(s,\psi(s)\big),\\
  \psi(t_2)=\phi_{t_1}(t_2,w),
  \end{cases}
\]
the flow satisfies the composition rule
\begin{equation}\label{eq:times}
  \phi_{t_1}(t_3, w) = \phi_{t_2}(t_3, \phi_{t_1}(t_2, w)).
\end{equation}
In particular, for every $s\in[0,t]$, the map
$w\mapsto \phi_t(s,w)$ is a homeomorphism of $\R^d$, with inverse $w\mapsto \phi_s(t,w)$, that is,
\begin{equation}\label{eq:inverse-flow}
  \phi_t(s,\cdot)^{-1}=\phi_s(t,\cdot).
\end{equation}
Moreover, by Theorem~\ref{thm:LLN} and Assumption~\ref{assump:A1}, the advection field $(s,w)\mapsto F(s,w)$ belongs to $\C^1([0,t]\times\R^d)$ and is $\C^\infty$ with respect to $w$. Hence the map $(t,s,w)\mapsto\phi_t(s,w)$ is of class $\C^1$ (see, e.g.,~\cite[Chapter~5, Section~3, Theorem~3.1]{hartman-2002}), and
\begin{equation}\label{eq:Jacobian}
  \partial_t\phi_t(s,w)
  =
  -J_w\phi_t(s,w)F(t,w),
\end{equation}
where $J_w\phi_t(s,w)$ denotes the Jacobian matrix of
$w\mapsto\phi_t(s,w)$.
Furthermore, by~\cite[Chapter~5, Section~3, Theorem~4.1]{hartman-2002}, $w\mapsto\phi_t(s,w)$ is $\C^\infty$ and
$(t,s,w)\mapsto\partial_w^m\phi_t(s,w)$ is continuous for every $m\in\N^d$.
Here the subscript $w$ emphasizes that $\partial_w^m$ denotes differentiation only with respect to the spatial variable $w \in \R^d$.
Since all spatial derivatives of $F$ are uniformly bounded (by~\ref{assump:A1}), it follows by induction that for every $m\in\N^d$ such that $|m|>0$,
\begin{equation}\label{eq:partial-w}
  \sup_{s,r\in[0,t]}\sup_{w\in\R^d}
  |\partial_w^m\phi_s(r,w)|<\infty.
\end{equation}
Finally, since $F$ is bounded, identity~\eqref{eq:Jacobian} and the first-order bounds of~\eqref{eq:partial-w} give
\begin{equation}\label{eq:partial-t}
  \sup_{s,r\in[0,t]}\sup_{w\in\R^d}
  |\partial_s\phi_s(r,w)|<\infty.  
\end{equation}

\subsection{Duhamel formula}
\label{sec:Duhamel}
Fix again $t \ge 0$ and let $g\in\C^1\big([0,t]\times \R^d\big)$. Any sufficiently smooth solution (say, e.g., with regularity $\C^1\big([0,t] \times \R^d\big)$) to
\eqref{eq:frozen} satisfies along characteristics, for every $s \in [0,t]$ and $w \in \R^d$,
\[
  \begin{cases}
    \dfrac{d}{ds}\Big[f\big(s,\phi_t(s,w)\big)\Big]
    =
    -Ag\big(s,\phi_t(s,w)\big), \\[0.8em]
    f\big(t,\phi_t(t,w)\big) =\varphi(w).
  \end{cases}
\]
Integrating over $[s,t]$ yields
\[
  f(s, \phi_t(s,w)) = \varphi(w) + \int_s^t Ag\big(r,\phi_t(r,w)\big)\,dr.
\]
Recall that the inverse of the mapping $w \mapsto \phi_t(s,w)$ is $w \mapsto \phi_s(t,w)$ (see~\eqref{eq:inverse-flow}). Using also $\phi_t(r, \phi_s(t,w)) = \phi_s(r,w)$ (by~\eqref{eq:times}), we obtain
\begin{equation}\label{eq:Duhamel}
  f(s,w)
  =
  \varphi\big(\phi_s(t,w)\big)
  +
  \int_s^t Ag\big(r,\phi_s(r,w)\big)\,dr.
\end{equation}
This is the Duhamel formula for sufficiently smooth solutions associated with the transport part of \eqref{eq:backward-transport-equation}. 

\begin{definition}\label{def:mild-T}
Let $t\ge0$ and let $\varphi\in\H^{J_0,j_0}(\R^d)$. We say that
$f\in\C\big([0,t],\H^{J_0,j_0}(\R^d)\big)$ is a \emph{mild solution}
of~\eqref{eq:T} on $[0,t]$ with terminal datum $\varphi$ if, for every
$s\in[0,t]$,
\begin{equation}\label{eq:mild-T}
  f(s,\cdot)
  =
  \varphi\big(\phi_s(t,\cdot)\big)
  +
  \int_s^t
  Af\big(r,\phi_s(r,\cdot)\big)\,dr
  \quad\text{in }\H^{J_0,j_0}(\R^d).
\end{equation}
\end{definition}

The preceding definition is meaningful in the Sobolev sense for the following reasons. 
First, recall by Proposition~\ref{prop:embedding} that every element of $\H^{J_0,j_0}(\R^d)$ can be identified with its unique representative in $\C^{1,j_0}(\R^d)$. In particular, if $g\in \C([0,t],\H^{J_0,j_0}(\R^d))$, then $g(s,\cdot)$ and
$\nabla g(s,\cdot)$ are well-defined pointwise through this representative.
Moreover, since $\bar\mu_s\in \mathcal{P}_\gamma(\R^d)$ by Theorem~\ref{thm:LLN}, $j_0\le \gamma$ (see~\eqref{eq:exponents}), and $\sigma_*(x,\cdot)\in \C_b^\infty(\R^d)$ by~\ref{assump:A1}, the quantity $\big\langle \nabla g(s)\cdot \nabla \sigma_*(x,\cdot),\bar\mu_s\big\rangle$ is well defined for every $s\in[0,t]$ and every $x\in\X$. Hence the expression $Ag(s,\cdot)$ and the Duhamel formula~\eqref{eq:Duhamel} are meaningful for Sobolev-valued functions $g$. The estimates proved in Lemma~\ref{lem:regularity} below show in particular that the two terms on the right-hand side of~\eqref{eq:Duhamel} belong to $\H^{J_0,j_0}(\R^d)$, and Step~1 in the proof of Theorem~\ref{thm:pde} shows that the right-hand side defines an element of
$\C\big([0,t],\H^{J_0,j_0}(\R^d)\big)$.
Thus, solving the backward equation~\eqref{eq:T} in $\C\big([0,t],\H^{J_0,j_0}(\R^d)\big)$ amounts to finding a fixed point, in this space, of the map $\Phi$ suggested by~\eqref{eq:Duhamel} and introduced in~\eqref{eq:Phi} below. 

When the terminal datum is smooth, this solution will be shown in Step~3 of the proof of Theorem~\ref{thm:pde} to belong to 
\[
  \C^1\big([0,t],\H^{J_0,j_0}(\R^d)\big) \cap \C^1\big([0,t]\times\R^d\big)
\]
and therefore to solve the backward equation pointwise. 

The following lemma provides the two stability estimates needed to turn~\eqref{eq:Duhamel} into a well-defined fixed-point problem on $\C\big([0,t],\H^{J_0,j_0}(\R^d)\big)$.

\begin{lemma}\label{lem:regularity}
There is a constant $C >0$ (depending only on the final time $t \ge 0$) such that:
\begin{enumerate}[label={\upshape(\roman*)}]
    \item\label{lem:regularityi} For every $h\in \H^{J_0,j_0}(\R^d)$ and every $0\le s\le r\le t$,
    \begin{equation}
    \label{eq:composition}
      \|h\circ \phi_s(r,\cdot)\|_{\H^{J_0,j_0}}
      \le C\,\|h\|_{\H^{J_0,j_0}}.
    \end{equation}

    \item\label{lem:regularityii} For every $g\in \C([0,t],\H^{J_0,j_0}(\R^d))$ and every $s\in[0,t]$,
    \begin{equation}
    \label{eq:A}
      \|Ag(s,\cdot)\|_{\H^{J_0,j_0}}
      \le C\,\|g(s)\|_{\H^{J_0,j_0}}.
    \end{equation}
\end{enumerate}
\end{lemma}

\begin{proof}
\ref{lem:regularityi}. By density, it is enough to consider
$h\in\C_c^\infty(\R^d)$. Fix $0 \le s \le r \le t$ and let $m\in\N^d$ with $|m|\le J_0$. By the chain rule, $\partial_w^m(h\circ\phi_s(r,\cdot))$ is a finite sum of terms of
the form
\[
  \big(\partial^\ell h\big)(\phi_s(r,\cdot))
  \prod_{i=1}^{|\ell|}
  \partial_w^{m_i}\phi_s(r,\cdot),
\]
where $1\le |\ell|\le |m|$, $m_1,\ldots,m_{|\ell|}\neq 0$, and
$|m_1|+\cdots+|m_{|\ell|}|=|m|$, together with the trivial term
$h\circ\phi_s(r,\cdot)$ when $m=0$. By~\eqref{eq:partial-w}, all products of derivatives of the flow are bounded uniformly for $0 \le s \le r \le t$. Hence there exists a constant $C > 0$, depending only on the final time $t$, such that
\begin{equation}\label{eq:chain-rule}
   \int_{\R^d}
 \frac{\big|\partial_w^m(h\circ\phi_s(r,w))\big|^2}
 {1+|w|^{2j_0}}\,dw
 \le
 C
 \sum_{|\ell|\le |m|}
 \int_{\R^d}
 \frac{\big|(\partial_w^\ell h)(\phi_s(r,w))\big|^2}
 {1+|w|^{2j_0}}\,dw.
\end{equation}
Since $\phi_s(r,\cdot)$ is a $\C^1$-diffeomorphism with inverse $\phi_r(s,\cdot)$ (see~\eqref{eq:inverse-flow}), the change of variables $u=\phi_s(r,w)$ gives for every $|\ell| \le |m|$,
\begin{equation}\label{eq:change-of-variables}
    \int_{\mathbb R^d}
  \frac{\left|(\partial^\ell h)(\phi_s(r,w))\right|^2}
       {1+|w|^{2j_0}}\,dw
  =
  \int_{\mathbb R^d}
  \frac{|\partial^\ell h(u)|^2}
       {1+|\phi_r(s,u)|^{2j_0}}
  \left|\det J_u\phi_r(s,u)\right|\,du.
\end{equation}
By the first-order bounds in~\eqref{eq:partial-w} together with the inverse-flow formula~\eqref{eq:inverse-flow}, the determinant $\left|\det J_u\phi_r(s,u)\right|$ is uniformly bounded for $0\le s\le r\le t$ and $u\in\R^d$. Using also the continuity of $(a,b)\mapsto\phi_a(b,0)$ on the compact set $\{0\le a\le b\le t\}$, we have for every $0\le a\le b\le t$ and $z\in\R^d$,
\begin{equation}\label{eq:growth}
  |\phi_a(b,z)|
  \le  |\phi_a(b,0)|
     + \sup_{0\le a\le b\le t} \|J_z\phi_{a}(b,\cdot)\|_\infty |z|
  \le C(1+|z|),  
\end{equation}
where the constant $C>0$ may change from line to line and depends only on the final time $t$. Applying~\eqref{eq:growth} with $a = s$, $b=r$, $z=\phi_r(s,u)$, and using again~\eqref{eq:inverse-flow}, we get
\[
  |u|
  =
  |\phi_s(r,\phi_r(s,u))|
  \le C\bigl(1+|\phi_r(s,u)|\bigr).
\]
Combining this with~\eqref{eq:chain-rule} and~\eqref{eq:change-of-variables}, and summing over all $|m|\le J_0$ yields
\[
  \begin{aligned}
    \|h\circ \phi_s(r,\cdot)\|^2_{\H^{J_0,j_0}}
    &=
    \sum_{|m| \le J_0} \int_{\R^d}
    \frac{\left|\partial_w^m(h\circ\phi_s(r,w))\right|^2}
        {1+|w|^{2j_0}}\,dw\\
    &\le
    C
    \sum_{|m| \le J_0}
    \sum_{|\ell|\le |m|}
    \int_{\R^d}
    \frac{|\partial^\ell h(u)|^2}{1+|u|^{2j_0}}\,du
    =C\,\|h\|^2_{\H^{J_0,j_0}}.
  \end{aligned}
\]
This proves~\eqref{eq:composition}.

\medskip
\noindent
\ref{lem:regularityii}. Let
$g\in\C([0,t],\H^{J_0,j_0}(\R^d))$ and fix $s\in[0,t]$. 
For every multi-index $m\in\N^d$ with $|m|\le J_0$, 
\[
  \partial_w^m Ag(s,w)
  =
  -\alpha\,\E_\pi\Big[
  \big\langle \nabla g(s)\cdot\nabla\sigma_*(x,\cdot),\bar\mu_s\big\rangle
  \partial_w^m\sigma_*(x,w)
  \Big].
\]
By the embedding $\H^{J_0,j_0}(\R^d)\hookrightarrow \C^{1,j_0}(\R^d)$ from Proposition~\ref{prop:embedding}, for all $u\in\R^d$, we have
\[
  |\nabla g(s,u)|
  \le C\|g(s)\|_{\H^{J_0,j_0}}(1+|u|^{j_0}).
\]
Since $j_0\le\gamma$, Theorem~\ref{thm:LLN} ensures that
\begin{equation}\label{eq:moments}
  \sup_{s\in[0,t]}\langle 1+|\cdot|^{j_0},\bar\mu_s\rangle<\infty.
\end{equation}
Since $\partial_w^m\sigma_*$ is uniformly bounded over $\X \times \R^d$ for every $|m|\le J_0$ (by~\ref{assump:A1}), the above estimates imply that
\[
  |\partial_w^m Ag(s,w)|
  \le
  C\,\|g(s)\|_{\H^{J_0,j_0}},
\]
for some constant $C > 0$.
Moreover, as $j_0>d/2$, we have
\[
  \int_{\R^d}\frac{dw}{1+|w|^{2j_0}}<\infty.
\]
Hence, after possibly changing the value of $C$,
\[
  \|Ag(s,\cdot)\|^2_{\H^{J_0,j_0}}
  =
  \sum_{|m|\le J_0}
  \int_{\R^d}
  \frac{|\partial_w^m Ag(s,w)|^2}{1+|w|^{2j_0}}\,dw
  \le C\|g(s)\|_{\H^{J_0,j_0}}^2,
\]
which proves~\eqref{eq:A}.
\end{proof}

\subsection{Proof of Theorem~\ref{thm:pde}}

Fix $t\ge 0$, and let $\varphi\in \H^{J_0,j_0}(\R^d)$. For
$g\in \C\big([0,t],\H^{J_0,j_0}(\R^d)\big)$ and $s\in[0,t]$, set
\begin{equation}\label{eq:Phi}
  \Phi g(s,\cdot)
  =
  \varphi\big(\phi_s(t,\cdot)\big)
  +
  \int_s^t
  Ag\big(r,\phi_s(r,\cdot)\big)\,dr.
\end{equation}
By the Duhamel formulation~\eqref{eq:Duhamel}, solving~\eqref{eq:T} (in the sense of Definition~\ref{def:mild-T}) in $\C\big([0,t],\H^{J_0,j_0}(\R^d)\big)$ is equivalent to finding a fixed point of $\Phi$. We consider the sup norm on $\C\big([0,t],\H^{J_0,j_0}(\R^d)\big)$, given by
\[
  \|g\|_{\C([0,t],\H^{J_0,j_0}(\R^d))}
  =
  \sup_{s\in[0,t]}\|g(s)\|_{\H^{J_0,j_0}}.
\]

\medskip
\noindent
\emph{Step 1.}
Let us first check that $\Phi$ maps $\C\big([0,t],\H^{J_0,j_0}(\R^d)\big)$ into itself. Let us emphasize that, throughout this step, the terminal datum $\varphi$ is only assumed to belong to $\H^{J_0,j_0}(\R^d)$. By Lemma~\ref{lem:regularity}, there exists a constant $C>0$ depending only on $t$ such that, for every $s\in[0,t]$,
\[
  \|\Phi g(s)\|_{\H^{J_0,j_0}}
  \le
  C\|\varphi\|_{\H^{J_0,j_0}}
  + C^2\int_s^t \|g(r)\|_{\H^{J_0,j_0}}\,dr.
\]
In particular, up to renaming $C$,
\[
  \sup_{s\in[0,t]}\|\Phi g(s)\|_{\H^{J_0,j_0}}
  \le
  C\|\varphi\|_{\H^{J_0,j_0}}
  + C \sup_{s\in[0,t]}\|g(s)\|_{\H^{J_0,j_0}}.
\]
It remains to check continuity in the time variable. 
We first check that if
$h_n\to h$ in $\H^{J_0,j_0}(\R^d)$ and
$(s_n,t_n)\to(s, t)$ in $[0,t]^2$, then
\begin{equation}\label{eq:continuity}
  h_n\circ \phi_{s_n}(t_n,\cdot)
  \longrightarrow
  h\circ \phi_s(t,\cdot)
  \quad\text{in }\H^{J_0,j_0}(\R^d).
\end{equation}
Indeed, on the one hand, by Lemma~\ref{lem:regularity}~\ref{lem:regularityi}, 
\[
  \|(h_n - h) \circ \phi_{s_n}(t_n, \cdot)\|_{\mathcal{H}^{J_0,j_0}} \le C\,\|h_n - h\|_{\mathcal{H}^{J_0,j_0}},
\]
for some constant $C > 0$. This term vanishes as $n \to \infty$ since $h_n \to h$ in $\H^{J_0,j_0}(\R^d)$. On the other hand, for any $h^* \in \C_c^\infty(\R^d)$, 
\[
  \|(h_n - h^*) \circ \phi_{s_n}(t_n, \cdot)\|_{\H^{J_0,j_0}} \le \|h_n - h^*\|_{\H^{J_0,j_0}},
\]
and, by the dominated convergence theorem,
\[
  \|h^* \circ \phi_{s_n}(t_n, \cdot) - h^* \circ \phi_{s}(t, \cdot)\|_{\H^{J_0,j_0}} \longrightarrow 0.
\]
The convergence~\eqref{eq:continuity} follows by density of $\C_c^\infty(\R^d)$ in $\H^{J_0,j_0}(\R^d)$.
Applying~\eqref{eq:continuity} with $h_n=h=\varphi$ shows that
$s\mapsto \varphi\circ\phi_s(t,\cdot)$ is continuous in
$\H^{J_0,j_0}(\R^d)$. 
We next note that the map
\[
  r \in [0,t] \longmapsto Ag(r,\cdot) \in \H^{J_0,j_0}(\R^d)
\]
is continuous. Indeed, since $g \in \C\big([0,t], \H^{J_0,j_0}(\R^d)\big)$, $\H^{J_0,j_0}(\R^d)\hookrightarrow \C^{1,j_0}(\R^d)$ from Proposition~\ref{prop:embedding} and the regularity of $\bar\mu$ from Theorem~\ref{thm:LLN},~\eqref{eq:moments}, and~\ref{assump:A1}, we have
for every $m \in \N^d$ and every $w\in\R^d$,
\[
  \partial_w^m Ag(r_n,w)\longrightarrow \partial_w^m Ag(r,w),
\]
and then, since $j_0 > d/2$,
\[
  Ag(r_n,\cdot)\longrightarrow Ag(r,\cdot)
  \quad\text{in }\H^{J_0,j_0}(\R^d).
\]
Combining this with~\eqref{eq:continuity} shows that $(s,r)\longmapsto Ag(r,\cdot)\circ\phi_s(r,\cdot)$ is continuous from the compact set $\{0\le s\le r\le t\}$ into $\H^{J_0,j_0}(\R^d)$. Hence
\[
  s \in [0,t] \longmapsto
  \int_s^t Ag(r,\cdot)\circ\phi_s(r,\cdot)\,dr
\]
is continuous in $\H^{J_0,j_0}(\R^d)$.
Hence $\Phi g\in \C\big([0,t],\H^{J_0,j_0}(\R^d)\big)$.

\medskip
\noindent
\emph{Step 2.}
We now prove that $\Phi$ has a unique fixed point in $\C\big([0,t],\H^{J_0,j_0}(\R^d)\big)$ by applying the contraction principle to a sufficiently high iterate of $\Phi$. Indeed, by linearity of $A$ and Lemma~\ref{lem:regularity}~\ref{lem:regularityii}, there exists a constant $C>0$ (depending only on $t$) such that, for every $g,h\in \C\big([0,t],\H^{J_0,j_0}(\R^d)\big)$ and $s\in[0,t]$, 
\begin{equation}\label{eq:contraction}
  \begin{aligned}
  \|\Phi g(s)-\Phi h(s)\|_{\H^{J_0,j_0}}
  &\le
  C\int_s^t \|g(r)-h(r)\|_{\H^{J_0,j_0}}\,dr.
\end{aligned}
\end{equation}
Iterating~\eqref{eq:contraction} thus yields, for every $k\ge1$,
\[
  \sup_{s\in[0,t]}
  \|\Phi^k g(s)-\Phi^k h(s)\|_{\H^{J_0,j_0}}
  \le
  \frac{(Ct)^k}{k!}\,
  \sup_{s\in[0,t]}\|g(s)-h(s)\|_{\H^{J_0,j_0}}.
\]
Hence, for $k$ large enough, $\Phi^k$ is a contraction on
$\C\big([0,t],\H^{J_0,j_0}(\R^d)\big)$, which implies that $\Phi$ has a unique fixed point in $\C\big([0,t],\H^{J_0,j_0}(\R^d)\big)$, denoted by $f^\varphi$.

\medskip
\noindent
\emph{Step 3.} 
In this step, we show that when the terminal datum $\varphi$ is smooth, 
the (unique) solution $f$ to $\Phi f=f$ in $\C([0,t],\H^{J_0,j_0}(\R^d))$
belongs in fact to 
\[
  \C^1\big([0,t],\H^{J_0,j_0}(\R^d)\big) \cap \C^1\big([0,t]\times\R^d\big).
\]
To this end, assume now that $\varphi\in\C_b^\infty(\R^d)$. 
Since, for every $r\in[0,t]$ and $w\in\R^d$, the map
$s\mapsto \phi_s(r,w)$ is of class $\C^1$, and since both the activation function
$\sigma_*$ and the terminal datum $\varphi$ are smooth, it follows from
the definition of $\Phi$ that, for every $w\in\R^d$, the map
\[
  s \in [0,t] \longmapsto
  \Phi g(s,w)
  =
  \varphi\big(\phi_s(t,w)\big)
  +
  \int_s^t
  Ag\big(r,\phi_s(r,w)\big)\,dr
\]
is of class $\C^1$. More precisely, using the regularity of the flow, the
boundedness of its derivatives~\eqref{eq:partial-w}--\eqref{eq:partial-t},
and Assumption~\ref{assump:A1}, we have $\Phi g\in \C^1\big([0,t]\times\R^d\big)$. Combined with~\eqref{eq:partial-t} and the boundedness of all derivatives of $\varphi$ and $\sigma_*$, one easily deduces that the function $\partial_s\Phi g \in \C\big([0,t],\H^{J_0,j_0}(\R^d)\big)$.
Consequently, for every
$g\in\C\big([0,t],\H^{J_0,j_0}(\R^d)\big)$, we have
\[
  \Phi g
  \in
  \C^1\big([0,t],\H^{J_0,j_0}(\R^d)\big)
  \cap
  \C^1\big([0,t]\times\R^d\big).
\]

\medskip
\noindent
\emph{Step 4.} It remains to prove the desired  continuity with respect to the terminal datum. 
Using the fixed-point formulation~\eqref{eq:Phi} for two terminal data $\varphi,\psi\in\H^{J_0,j_0}(\R^d)$ together with Lemma~\ref{lem:regularity}, we obtain, for every $s\in[0,t]$,
\[
  \|f^\varphi(s)-f^\psi(s)\|_{\H^{J_0,j_0}}
  \le
  C\,\|\varphi-\psi\|_{\H^{J_0,j_0}}
  + C\int_s^t
  \|f^\varphi(r)-f^\psi(r)\|_{\H^{J_0,j_0}}\,dr.
\]
By Gronwall's lemma, it follows that for all $t\ge0$, 
\[
  \sup_{s\in[0,t]}
  \|f^\varphi(s)-f^\psi(s)\|_{\H^{J_0,j_0}}
  \le
  C\|\varphi-\psi\|_{\H^{J_0,j_0}}.
\]
Thus the solution map $\varphi\in\H^{J_0,j_0}(\R^d)
\mapsto f^\varphi\in \C\big([0,t],\H^{J_0,j_0}(\R^d)\big)$
is linear and Lipschitz-continuous.
The proof of Theorem~\ref{thm:pde} is complete. \qed

\section{Numerical experiments}
\label{sec:numerics}

In this section, we illustrate the variance formula of Theorem~\ref{thm:main} on the simple one-dimensional regression problem taken from~\cite{mei-montanari-nguyen-2018}:
\begin{equation}\label{eq:model}
  y = x^3 + \varepsilon,
 \qquad
 x \sim \mathcal{U}([-1,1]), \
 \qquad
 \varepsilon \sim \mathcal{N}(0,\sigma_\varepsilon^2),
\end{equation}
where $\varepsilon$ is independent of $x$ and the noise level is chosen as $\sigma_\varepsilon = 0.05$.
We specialize the two-layer model~\eqref{eq:network} to the case $d = 2$ by writing $\theta^i = (w_i,b_i)\in \R^2$ and $\sigma_*(x, \theta^i)=\tanh(w_i x+b_i)$, $i \in \{1, \ldots, N\}$, so that
\[
 \hat{y}_{\theta}(x) = \frac1N\sum_{i = 1}^N \tanh(w_i x + b_i).
\]

Fix a prediction point $x\in[-1, 1]$. For $(w,b) \in \R^2$, let 
\[
  \varphi_{x}(w, b) = \sigma_*\big(x,(w, b)\big) = \tanh(wx+b).
\]
By Theorems~\ref{thm:LLN} and~\ref{thm:CLT} recalled in Section~\ref{sec:LLN-CLT}, at training time $\lfloor Nt\rfloor$, $t \ge 0$, we have
\begin{equation}\label{eq:fluctuations}
  \hat{y}_{\theta_{\lfloor Nt\rfloor}}(x)
  =
 \langle \varphi_x,\mu_t^N \rangle
 \approx
 \langle \varphi_x,\bar\mu_t \rangle
 +
 \frac1{\sqrt{N}} \langle \varphi_x, \eta_t^* \rangle.
\end{equation}
Hence, the fluctuations of the network output around its mean-field limit
are of order $N^{-1/2}$, and the corresponding non-degenerate finite-width
variance scale is
\begin{equation}\label{eq:finite-width-variance}
 V_N(x)
 =
 N\,\Var\big(\hat{y}_{\theta_{\lfloor Nt\rfloor}}(x)\big).
\end{equation}
In view of the approximation~\eqref{eq:fluctuations}, the corresponding large-width limiting output variance is
\[
 V^*(x)
 =
 \Var\big(\langle \varphi_x, \eta_t^* \rangle\big).
\]
By Theorem~\ref{thm:main}, this quantity is given by the variance formula~\eqref{eq:variance}, where $f^{\varphi_x}$ is obtained by solving the backward equation~\eqref{eq:T} on $[0,t]$ with terminal datum $\varphi_x$. Accordingly, for each prediction point $x \in [-1, 1]$, we compare a Monte Carlo (or \emph{deep ensemble}) estimator $\widehat{V}(x)$ of $V_N(x)$, used as the finite-$N$ benchmark, with the PDE-based approximation $\widehat{V}^*(x)$ of its large-width limit $V^*(x)$.

\subsection{Monte Carlo reference}

As a finite-width benchmark, we estimate the quantity
$x \mapsto V_N(x)$ directly using a Monte Carlo simulation. More precisely, we simulate $M = 20{,}000$ neural networks, with initial parameters $\theta_0^i$ sampled independently from $\mathcal{N}(0,0.01\,I_2)$ and trained independently with noiseless mini-batch~\eqref{eq:SGD} under the same hyperparameters ($N = 2000$ neurons, mini-batch size $B \equiv 20$, learning rate $\alpha = 600$) up to time $t = 10$ (that is, for $\lfloor Nt \rfloor = 20{,}000$ iterations). For each set of trained parameters 
\[
  \theta^{(m)} = \bigl(\theta_{\lfloor Nt \rfloor}^{1,m},\ldots,\theta_{\lfloor Nt \rfloor}^{N,m}\bigr) \in \R^{2N},
  \qquad m \in \{1, \ldots, M\},
\]
we evaluate the output $\hat{y}^{(m)}(x) = \hat{y}_{\theta^{(m)}}(x)$ on a fixed grid of $200$ input points $x\in[-1, 1]$. This provides, at each input, an empirical (unbiased) estimator of $V_N(x)$, given by
\begin{equation}\label{eq:VMC}
  \widehat{V}(x) = \frac{N}{M-1} \sum_{m=1}^M \bigg(\hat{y}^{(m)}(x) - \frac1M \sum_{\ell=1}^M \hat{y}^{(\ell)}(x)\bigg)^2.
\end{equation}

\begin{remark}
  To aggregate the $M$ independent Monte Carlo runs efficiently and allow for parallel implementations, we accumulate the empirical variance of $\hat{y}^{(m)}(x)$, $m \in \{1, \ldots, M\}$ at each prediction point $x$ using Welford's online algorithm~\cite{welford-1962}; 
  the corresponding update formulas and the estimation of the relative Monte Carlo error are detailed in Appendix~\ref{app:Monte-Carlo}. This method avoids storing the full collection of trained network outputs while still producing the \emph{exact} sample variance used in the comparison. With $M = 20{,}000$ independent trainings, the empirical relative error of the Monte Carlo variance estimator (see Appendix~\ref{app:Monte-Carlo}) is reduced to less than $1\%$. Therefore, the corresponding error bars are omitted in Figure~\ref{fig:variance} for readability.
\end{remark}

\subsection{PDE-based variance estimator}

Although the main text is stated for \emph{online} SGD (i.e., for batch size
one), the numerical experiment conducted here is performed in the mini-batch setting with constant mini-batch size $B\equiv20$. Adapting the results of Theorem~\ref{thm:main} to this setting is straightforward from the mini-batch statement of Theorem~\ref{thm:CLT} in~\cite{descours-guillin-michel-nectoux-2024}. The corresponding variance formula is given by a simple modification of~\eqref{eq:variance}, namely, for the output test function $\varphi_x=\sigma_*(x,\cdot)$, 
\[
 V^*(x)
 =
 V_0(x)+\frac{\alpha^2}{B} I(x)
\]
where
\[
  V_0(x)=\Var\big(f^{\varphi_x}(0,\theta_0)\big),
  \quad
  I(x) = \int_0^t \Var_\pi\bigl(Q_s[f^{\varphi_x}(s)]\bigr)\, ds,
\]
and $f^{\varphi_x}$ solves the backward transport equation~\eqref{eq:T} on $[0,t]$ with terminal datum
$\varphi_x$.

In the implementation, the backward solve and the accumulation of the time integral $I(x)$ are performed simultaneously, yielding the numerical estimator
\[
 \widehat{V}^*(x) = \widehat V_0(x)+\frac{\alpha^2}{B}\, \widehat I(x).
\]
Here $\widehat I(x)$ is obtained by summing
the discretized values of $\Var_\pi(Q_s[f^{\varphi_x}(s)])$ along the backward time loop, while $\widehat V_0(x)$
is computed separately as the empirical variance of $f^{\varphi_x}(0, \theta_0)$ over i.i.d.\ samples $\theta_0\sim\mathcal{N}(0, 0.01\, I_2)$.

The backward transport equation~\eqref{eq:T} is discretized on a uniform $(w,b)$-grid with mesh size $h=0.02$. Time integration on $[0,t]$ is performed backward via a semi-Lagrangian scheme with time step $\Delta s=6\cdot 10^{-4}$, based on the Duhamel representation~\eqref{eq:Duhamel}.
More precisely, for $v=(w,b)\in\R^2$ and $0 \le s \le r \le t$, let $\phi_s(r,v)$ denote the characteristic flow associated with the advection field $F$ in~\eqref{eq:Duhamel}. We denote the foot of the characteristic reaching $v$ at time $s$ as
\[
  v^* = \phi_s(s-\Delta s,v).
\]
The backward semi-Lagrangian update then reads
\[
 f(s-\Delta s,v) \approx f(s,v^*)+\Delta s\,Af(s,v^*),
 \quad
 s = t, t-\Delta s, \ldots, \Delta s.
\]
In practice, $v^*$ is approximated by an RK2 midpoint step:
\[
 v_{\frac12} = v+\frac{\Delta s}{2}\, F(s-\Delta s / 2, v),
 \qquad
 v^*\approx v+\Delta s\, F\big(s-\Delta s/2, v_{\frac12}\big).
\]

To approximate the mean-field trajectory $\bar{\mu}$ entering~\eqref{eq:T}, we use proxy neural-network trajectories generated with the same training horizon $t = 10$, mini-batch size $B \equiv 20$, and learning-rate $\alpha = 600$ as in the Monte Carlo benchmark, but with width $N_{\bar\mu} = 10{,}000$. 
Each proxy network is trained with fresh i.i.d.\ mini-batches drawn from the data model~\eqref{eq:model}, as in the benchmark. 
We then retain the whole empirical-measure trajectory of this single large-width
network as a numerical surrogate for $(\bar\mu_s)_{0\le s\le t}$, thereby
avoiding the direct discretization of the mean-field
equation~\eqref{eq:LLN}.

At each backward time step, the advection field $F$ and the source term $A$ are thus computed from the reconstructed neural-network trajectory. Spatial derivatives of $f$ on the grid are computed by centered finite differences (with Neumann boundary conditions at the edges of the computational domain). The quantities that must be evaluated at the current proxy mean-field particles, namely $\nabla f$, $F$, $A$, and $f$ itself at the footpoints, are obtained from the grid values by bilinear interpolation. For the expectations under $\pi$, we use a single i.i.d.\ sample $x_1, \dots, x_n\sim \mathcal{U}([-1, 1])$ of size $n = 10^4$, drawn once and reused across all runs.

This also clarifies the computational comparison between the direct Monte Carlo
benchmark and the PDE-based method used here.
Training $M$ independent networks of width $N$ up to time $t$ requires $\lfloor N t\rfloor$ SGD steps per network, each costing of order $O(B N)$, and therefore has cost of order $O(B N^2 t M)$. The PDE-based procedure replaces these $M$ independent trainings by one large proxy trajectory, followed by the backward solve. For one terminal datum, if the parameter grid has size $n_w \times n_b$ and the quadrature sample for $\pi$ has size $n_\pi$, the backward part has cost of order $O(n_\pi\,n_w\,n_b\, N_{\bar\mu}\,t)$ on the proxy time grid. 

\begin{remark}
  To reduce the computational cost of the PDE quadrature, we take advantage of the additive structure of the data model~\eqref{eq:model} and separate the deterministic regression signal from the observation noise. Namely, for $s\in[0,t]$ and $x\in[-1,1]$, we set
\[
 r_s(x) = x^3- \langle \sigma_*(x,\cdot), \bar\mu_s \rangle,
\]
and, for the current backward solution $f(s,\cdot)$,
\[
 G_s[f](x) = \bigl\langle \nabla f(s, \cdot)\cdot\nabla \sigma_*(x,\cdot), \bar\mu_s\bigr\rangle.
\]
Then, for $y=x^3+\varepsilon$ as in~\eqref{eq:model},
\[
 Q_s[f] =\bigl(r_s(x)+\varepsilon\bigr)\,G_s[f](x).
\]
Using the independence between $\varepsilon$ and $x$, together with $\E[\varepsilon]=0$, we obtain
\[
 \Var_{\pi}\bigl(Q_s[f]\bigr)
 =
 \Var_\pi\bigl(r_s(x)\,G_s[f](x)\bigr)
 +\sigma_\varepsilon^2\,\mathbb E_\pi\bigl[G_s[f](x)^2\bigr].
\]
This decomposition allows us to use the noiseless
samples $y_j=x_j^3$ in the PDE quadrature, and to add the observation-noise contribution analytically through the second term above when updating $\widehat I(x)$.
\end{remark}

\begin{figure}[ht]
 \centering
 \includegraphics[width=.9\textwidth]{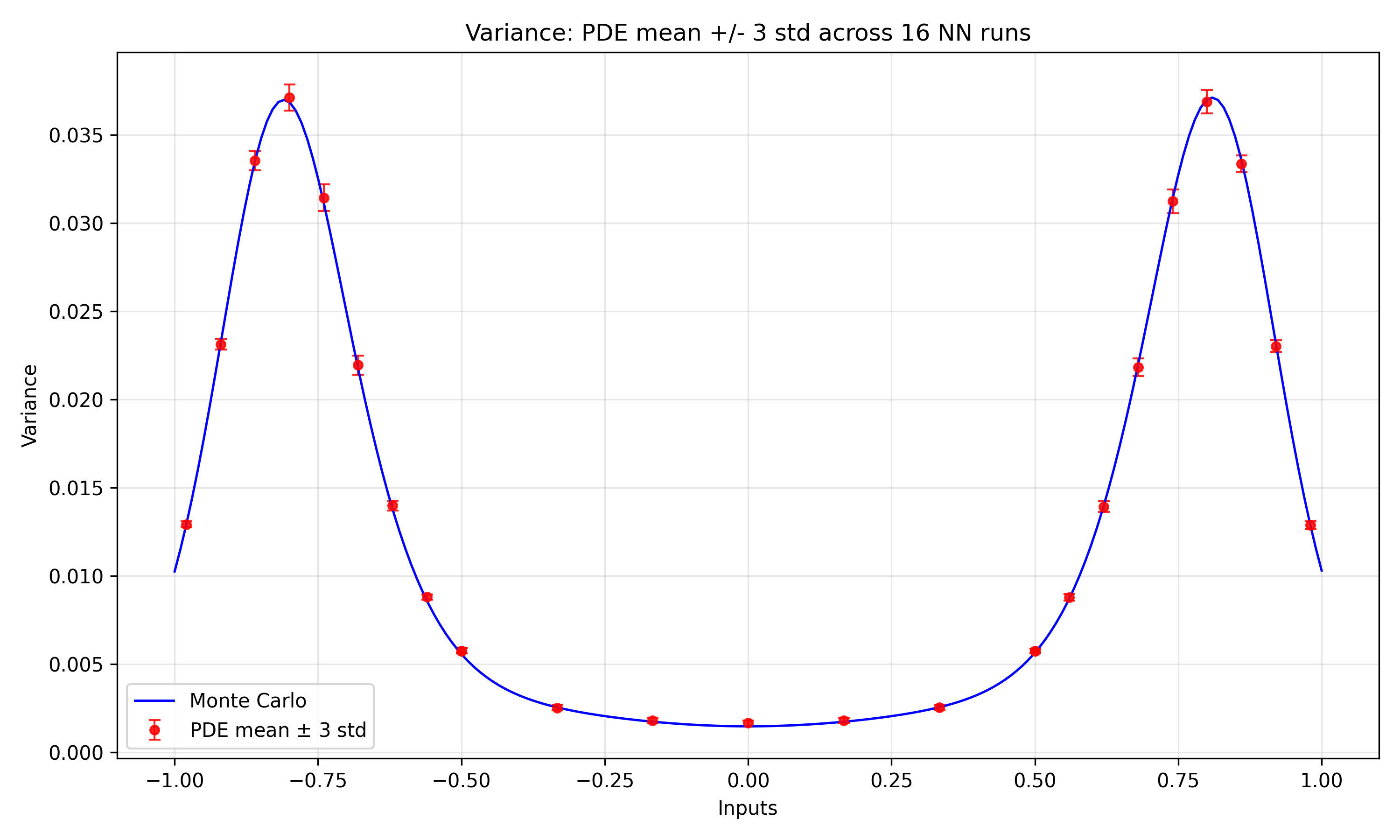}
 \caption{Comparison between the PDE-based approximation of the limiting variance $V^*(x)$ and the direct Monte Carlo estimate of $V_N(x)$. The blue curve is the Monte Carlo finite-width benchmark. The red markers correspond to the empirical mean of the PDE-based estimates over $16$ independent proxy trajectories at 25 non-uniform input points in $[-1,1]$, and the error bars represent $\pm 3$ empirical standard deviations across runs.}
 \label{fig:variance}
\end{figure}

\subsection{Comparison and discussion}

For the comparison shown in Figure~\ref{fig:variance}, the PDE estimator is evaluated on $25$ prediction points in $[-1, 1]$. To assess the sensitivity of the PDE approximation to the random neural-network trajectory used as a proxy for the mean-field dynamics $\bar\mu$, we repeat the whole PDE computation for $16$ independent runs.

Figure~\ref{fig:variance} shows a good qualitative agreement between the PDE-based approximation and the direct Monte Carlo benchmark. Moreover, the empirical dispersion across the $16$ PDE runs remains moderate compared with the scale of the signal, which suggests that the main variance profile is already captured reliably by the backward PDE. The remaining discrepancies are more plausibly attributable to the random proxy trajectory used to approximate the mean-field coefficients than to a systematic bias of the PDE discretization itself.

We stress, however, that the present experiment should be viewed as a first numerical validation of the theory on a simple one-dimensional example. In principle, one could discretize the deterministic mean-field equation~\eqref{eq:LLN} for $\bar\mu$, and then use this approximation to solve the associated backward equation~\eqref{eq:T}. If such a coupled deterministic procedure could be implemented accurately and efficiently, the variance formula~\eqref{eq:variance} derived here could provide the basis for an asymptotic uncertainty-quantification method for wide neural-network predictions. Developing such solvers in higher-dimensional and practically relevant regimes, and assessing their numerical accuracy and computational cost, appears to be an interesting direction for future work. 
The code used to produce the numerical results of this section is available at \url{https://github.com/pstos/Fluctuations}.

\noindent
\subsection*{Acknowledgements} 
A.D.\ was partially supported by the FLUTE project (EC grant No. 101095382) and by the French State under the France 2030 program (ANR-23-PEIA-005, REDEEM project). Part of this work was carried out while A.D. was a postdoctoral researcher at Inria Lille – Nord Europe.
G.L.\ was partially funded by the NumOpTES project (ANR-22-CE46-0005) of the French National Research Agency. This work was partially carried out while G.L. was a postdoctoral researcher at Institut de Mathématiques de Bordeaux, CNRS-UMR5251.
A.G.\ is supported by the ANR-23-CE-40003, Conviviality, and has benefited from a government grant managed by the Agence Nationale de la Recherche under the France 2030 investment plan ANR-23-EXMA-0001.
B.N.\ is supported by the grant IA20Nectoux from the Projet I-SITE Clermont CAP 20-25. 
P.S.\ is supported by the Projet I-SITE Clermont CAP 20-25.

\bibliographystyle{amsplain}
\bibliography{fluctuations}

\appendix

\section{Online Monte Carlo variance estimator and relative error}
\label{app:Monte-Carlo}

We detail the online computation of the Monte Carlo estimator used in Section~\ref{sec:numerics} for the 
finite-width variance
\[
  V_N(x)
  =
  N\Var\big(\widehat y_{\theta_{\lfloor Nt\rfloor}}(x)\big)
\]
with width $N \ge 1$, training time $t \ge 0$, and input $x \in [-1,1]$ fixed.
After $M$ independent trainings, we obtain, for each prediction point $x \in [-1,1]$, the independent output values
\[
  \widehat y^{(1)}(x),\,\ldots,\,\widehat y^{(M)}(x).
\]
For $1\le q\le M$ and $p\in\{2,3,4\}$, set
\[
  \overline{y}_q(x)
  =
  \frac1q\sum_{m=1}^q \widehat y^{(m)}(x),
  \qquad
  S_{p,q}(x)
  =
  \sum_{m=1}^q
  \big(
    \widehat y^{(m)}(x)-\overline{y}_q(x)
  \big)^p.
\]
With this notation, the Monte Carlo (unbiased) estimator of $V_N(x)$ in \eqref{eq:VMC} is
\begin{equation}\label{eq:VMC2}
  \widehat V(x)
  =
  \frac{N}{M-1}S_{2,M}(x).
\end{equation}

We now describe how the quantities $\overline y_q(x)$ and $S_{p,q}(x)$ are accumulated online. The recursion is initialized by
\[
  \overline{y}_1(x)=\widehat y^{(1)}(x),
  \qquad
  S_{2,1}(x)=S_{3,1}(x)=S_{4,1}(x)=0.
\]
Assume now that $\overline{y}_q(x)$ and $S_{p,q}(x)$, $p\in\{2,3,4\}$, have been computed after $q \ge 1$ runs. Given the next output value $\widehat{y}^{(q+1)}(x)$, set
\[
  \Delta_{q+1}(x)
  =
  \widehat y^{(q+1)}(x)-\overline{y}_q(x),
  \qquad
  \delta_{q+1}(x)
  =
  \frac{\Delta_{q+1}(x)}{q+1},
  \qquad
  T_{q+1}(x)
  =
  q\,\Delta_{q+1}(x)\delta_{q+1}(x).
\]
Hence, $\Delta_{q+1}(x)$ is the deviation of the new output from the previous running mean, $\delta_{q+1}(x)$ is the increment of the running mean, and $T_{q+1}(x)$ is the corresponding correction to the second centered moment.
With these increments, Welford's update for the empirical mean reads
\[
  \overline{y}_{q+1}(x)
  =
  \overline{y}_q(x)+\delta_{q+1}(x),
\]
and the second centered moment accumulator is updated by
\[
  S_{2,q+1}(x)
  =
  S_{2,q}(x)+T_{q+1}(x).
\]
If the third and fourth centered moments are also tracked, their online
updates are
\[
\begin{aligned}
  S_{3,q+1}(x)
  &=
  S_{3,q}(x)
  -3\delta_{q+1}(x)S_{2,q}(x)
  +T_{q+1}(x)\delta_{q+1}(x)(q-1),
  \\
  S_{4,q+1}(x)
  &=
  S_{4,q}(x)
  -4\delta_{q+1}(x)S_{3,q}(x)
  +6\delta_{q+1}(x)^2S_{2,q}(x)
  +T_{q+1}(x)\delta_{q+1}(x)^2(q^2-q+1).
\end{aligned}
\]
All these updates are performed independently at each prediction point $x$ on the output grid. At $q=M$, they yield the desired estimator $\widehat{V}(x)$ (see~\eqref{eq:VMC2}).

We next estimate the relative Monte Carlo error of this variance estimator. Set
\[
  v_N(x)
  =
  \frac{V_N(x)}{N}
  =
  \Var\big(\widehat y_{\theta_{\lfloor Nt\rfloor}}(x)\big),
\]
and let
\[
  m_{4,N}(x)
  =
  \E\Big[
    \big(
      \widehat y_{\theta_{\lfloor Nt\rfloor}}(x)
      -
      \E\big[\widehat y_{\theta_{\lfloor Nt\rfloor}}(x)\big]
    \big)^4
  \Big]
\]
be the fourth centered moment of the finite-width output at $x \in [-1,1]$. For $q\ge2$, the variance of the unbiased sample variance $S_{2,q}(x)/(q-1)$
satisfies
\[
  \Var\left(\frac{S_{2,q}(x)}{q-1}\right)
  =
  \frac1q
  \left(
    m_{4,N}(x)
    -
    \frac{q-3}{q-1}v_N(x)^2
  \right).
\]
This gives the plug-in estimator
\[
  \widehat{\Var}\left(\frac{S_{2,q}(x)}{q-1}\right)
  =
  \frac1q
  \left(
    \frac{S_{4,q}(x)}{q}
    -
    \frac{q-3}{q-1}
    \left(\frac{S_{2,q}(x)}{q-1}\right)^2
  \right).
\]
Therefore, the estimated relative standard error reported in Section~\ref{sec:numerics} is
\[
  \widehat{\mathrm{RSE}}_M(x)
  =
  \frac{
      \widehat{\Var}\big(S_{2,M}(x)/(M-1)\big)^{1/2}
  }{
    S_{2,M}(x)/(M-1)
  }.
\]

\end{document}